\documentclass[sigconf]{acmart}
\usepackage{amsmath,amssymb,amsfonts,dsfont,multirow,multicol,epsfig,url,array,makecell,balance,color,epstopdf}
\usepackage[ruled, vlined, linesnumbered]{algorithm2e}
\usepackage{hhline}
\usepackage{array}
\usepackage{enumerate}
\usepackage{enumitem}
\usepackage{subfigure}
\usepackage{booktabs}
\usepackage{xcolor,colortbl}
\setlist[itemize]{leftmargin=*}

\def\header{\vspace{2mm} \noindent}

\newcommand{\pushright}[1]{\ifmeasuring@#1\else\omit\hfill$\displaystyle#1$\fi\ignorespaces}
\newcommand{\pushleft}[1]{\ifmeasuring@#1\else\omit$\displaystyle#1$\hfill\fi\ignorespaces}

\def\strap{STRAP\xspace}
\def\e{\varepsilon}

\def\s{\hat{s}}

\def\s{\mathbf{s}}
\def\t{\mathbf{t}}

\def\inN{\mathcal{I}}
\def\outN{\mathcal{O}}
\def\din{d_{in}}
\def\dout{d_{out}}

\def\a{\alpha}

\def\P{P}
\def\sppr{\mathrm{SPPR}}

\def\ppr{\mathrm{PPR}}

\def\pib{\pi}

\def\rb{r}

\def\brmax{r_{max}}

\begin{document}
\copyrightyear{2019}
\acmYear{2019}
\setcopyright{acmcopyright}
\acmConference[KDD '19]{The 25th ACM SIGKDD Conference on Knowledge Discovery and Data Mining}{August 4--8, 2019}{Anchorage, AK, USA} \acmBooktitle{The 25th ACM SIGKDD Conference on Knowledge Discovery and Data Mining (KDD '19), August 4--8, 2019, Anchorage, AK, USA}
\acmPrice{15.00} \acmDOI{10.1145/3292500.3330860} \acmISBN{978-1-4503-6201-6/19/08}


\fancyhead{}

\title{Scalable Graph Embeddings via Sparse Transpose Proximities}
\subtitle{[Technical Report]}
\author{Yuan Yin}
\email{yinyuan@ruc.edu.cn}
  \affiliation{%
  \institution{Beijing Key Lab of Big
    Data Management and Analysis Method,  MOE Key Lab DEKE, School of Information}
   \city{Renmin University
    of China}
}

\author{Zhewei Wei}
%\authornotemark[1]
\authornote{Zhewei Wei is the corresponding author.}
\email{zhewei@ruc.edu.cn}
  \affiliation{%
  \institution{Beijing Key Lab of Big
    Data Management and Analysis Method,  MOE Key Lab DEKE, School of Information}
   \city{Renmin University
    of China}
}

\begin{abstract}
%  {\em SimRank} is a classic metric that measures the similarities of nodes in a graph. Given a node $u$ in graph $G =(V, E)$, a {\em single-source SimRank query} returns the SimRank similarities $s(u, v)$ between node $u$ and each node $v \in V$. This type of queries is widely used in web search and social networks, such as link prediction, web mining and spam detections. Existing work on  single-source SimRank queries, however, suffer from three major deficiencies. First, all previous algorithms incur query cost at least linear to the graph size $n$, which limits the scalability of these algorithms. Second, previous methods do not take into account the structure of input graphs. Third, most existing work is unable to provide empirical study for the accuracy of the algorithms on large graphs.

 Graph embedding learns low-dimensional representations for nodes in a graph and effectively preserves the graph structure. Recently, a significant amount of progress has been made toward this emerging research area. However, there are several fundamental problems that remain open. First, existing methods fail to preserve the out-degree distributions on directed graphs. Second, many existing methods employ random walk based proximities and thus suffer from conflicting optimization goals on undirected graphs. Finally, existing factorization methods are unable to achieve scalability and non-linearity simultaneously.
  
This paper presents an in-depth study on graph embedding techniques on both directed and undirected graphs. We analyze the fundamental reasons that lead to the distortion of out-degree distributions and to the conflicting optimization goals. We propose {\em transpose proximity}, a unified approach that  solves both problems. Based on the concept of transpose proximity, we design \strap, a factorization based graph embedding algorithm that achieves scalability and non-linearity simultaneously. \strap makes use of the {\em backward push} algorithm to efficiently compute the sparse  {\em Personalized PageRank (PPR)} as its transpose proximities. By imposing the sparsity constraint, we are able to apply non-linear operations to the proximity matrix and perform efficient matrix factorization to derive the embedding vectors. Finally, we present an extensive experimental study that evaluates the effectiveness of various graph embedding algorithms, and we show that \strap outperforms the state-of-the-art methods in terms of effectiveness and scalability. 
\end{abstract}

%%% Local Variables:
%%% mode: latex
%%% TeX-master: "paper"
%%% End:

\maketitle

% !TEX root=../ssppr_kdd17_newformat.tex
\vspace{0mm}
\section{Introduction} \label{sec:intro}
Graphs are a fundamental tool for understanding and modeling complex physical, social, informational, and biological systems. In recent years, graph embedding has drawn increasing attention from the academic fields due to its applications in various machine learning tasks.  The central idea of graph embedding is to learn a low-dimensional latent representation for nodes in the graph, such that the inherent properties and structures of the graph are preserved by the embedding vectors. These vectors can then be feed into well-studied machine learning methods in the vector space for common tasks on graphs such as classification, clustering, link prediction, and visualization.

In the past year, many methods have been proposed for learning node representations, and we summarize a few of recent ones in Table~\ref{tbl:proximities}. In general,  there are broadly two categories of approaches: methods which use random walks to learn the embedding vectors, and methods which use matrix factorization to directly derive the embedding vectors. Despite of their diversity, most of the existing methods adopt the following framework: 1) Determine a proximity measure $P(u,v)$; 2) Train embedding vector $\s_u$ for each node $u\in V$, such that $\s_u\cdot \s_v \sim P(u,v)$. For random walk methods,  $\s_u$'s are trained by skip-gram model~\cite{mikolov2013distributed} with negative sampling or hierarchical softmax; For factorization methods, $\s_u$'s  are directly derived from singular value decomposition (SVD) or eigen-decomposition.
Recently, \cite{tsitsulin2018verse}
and ~\cite{zhou2017scalable}  propose  that in order to  capture the asymmetry of directed graphs, we should train two vectors $\s_u$ and $\t_u$ as content/contextrepresentations, and thus the goal becomes to train $\s$ and $\t$\ such that $\s_u\cdot \t_v \sim P(u,v)$ for any $u,v\in V$. 

\begin{table}[!t]
\centering
\begin{small}
\renewcommand{\arraystretch}{1.2}
\vspace{0mm}
\caption{Proximities used by existing methods.} \label{tbl:proximities}
\vspace{-2mm}
\begin{tabular}{|p{0.6in}|p{1.8in}|p{0.65in}|}
    \hline
   {\bf Method}&  {\bf Proximity}&  {\bf Category}\\

    \hline
  DeepWalk~\cite{perozzi2014deepwalk}  &$\s_u\cdot \s_v \sim$ probability that a truncated random walk from $u$ visits $v$
             %\Pr[\textrm{Random walk from } u \textrm{ visit } v \textrm{ in } t \textrm{ steps.}]
                                    & Random Walk\\
  Node2Vec~\cite{grover2016node2vec} &$\s_u\cdot \s_v \sim$ probability that a truncated  2nd order random walk from $u$ visits $v$ & Random Walk\\
  LINE~\cite{tang2015line}  &$\s_u\cdot  \s_v \sim $ Adjacency relation between $u$ and $v$& Random Walk\\
  APP~\cite{zhou2017scalable} &$\s_u\cdot  \t_v \sim \ppr(u,v)$& Random Walk\\
  VERSE~\cite{tsitsulin2018verse} &$\s_u\cdot  \t_v \sim$ $\ppr(u,v)$, $ \mathrm{SimRank}(u,v)$& Random Walk\\
  HOPE~\cite{ou2016asymmetric} &$\s_u\cdot  \t_v \sim $ $\ppr(u,v), \mathrm{Katz}(u,v)$& Factorization\\
  AROPE~\cite{zhang2018arbitrary} &$\s_u\cdot  \t_v \sim $ Higher order proximity of form $\sum_{i=1}^qw_iA^i$& Factorization\\
    \hline
\end{tabular}
\vspace{-7mm}
\end{small}
\end{table}

\begin{figure*}[!t]
\begin{small}
 \centering
   \vspace{-4mm}
%    \begin{footnotesize}
   \includegraphics[height=28mm]{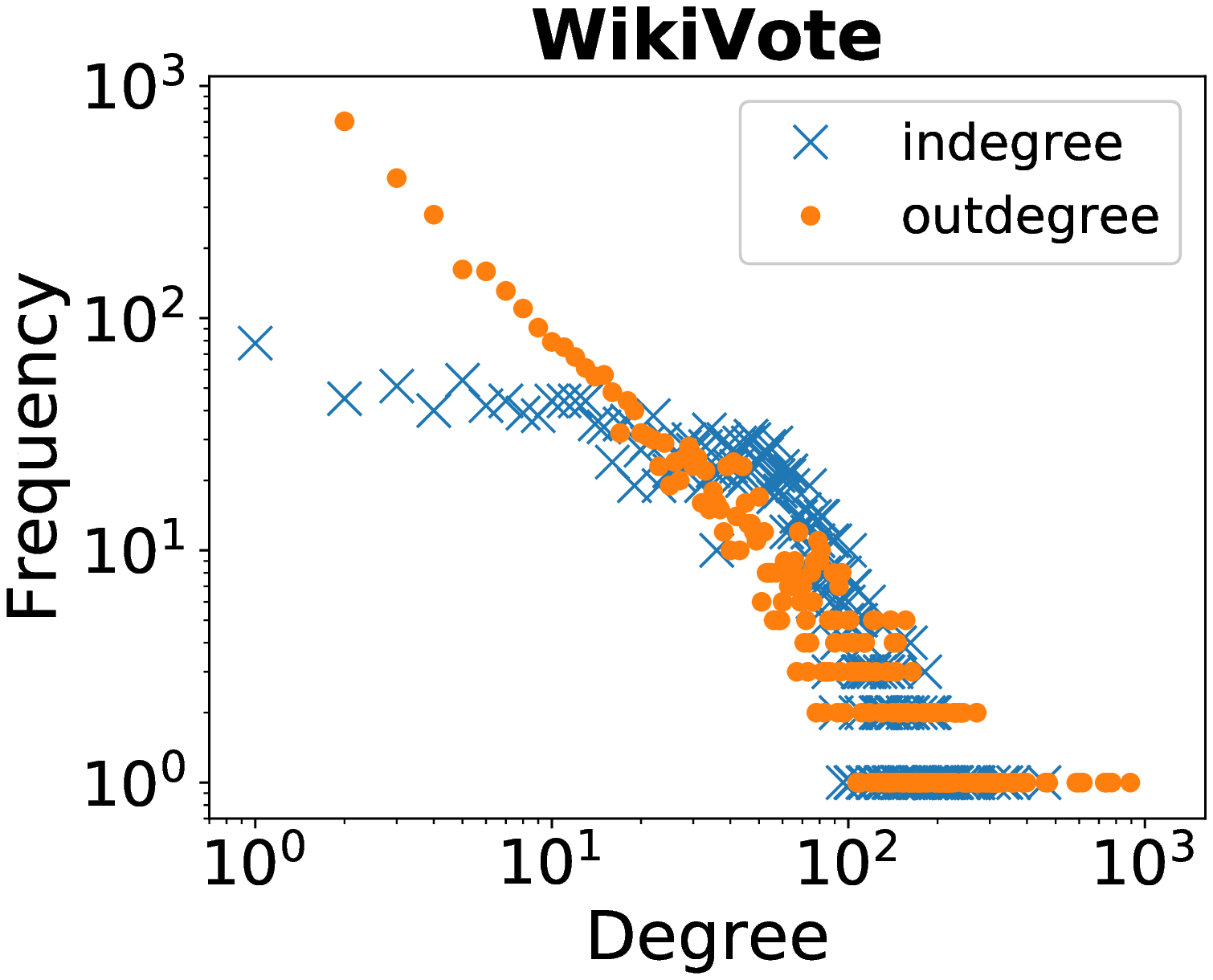}
 \includegraphics[height=28 mm]{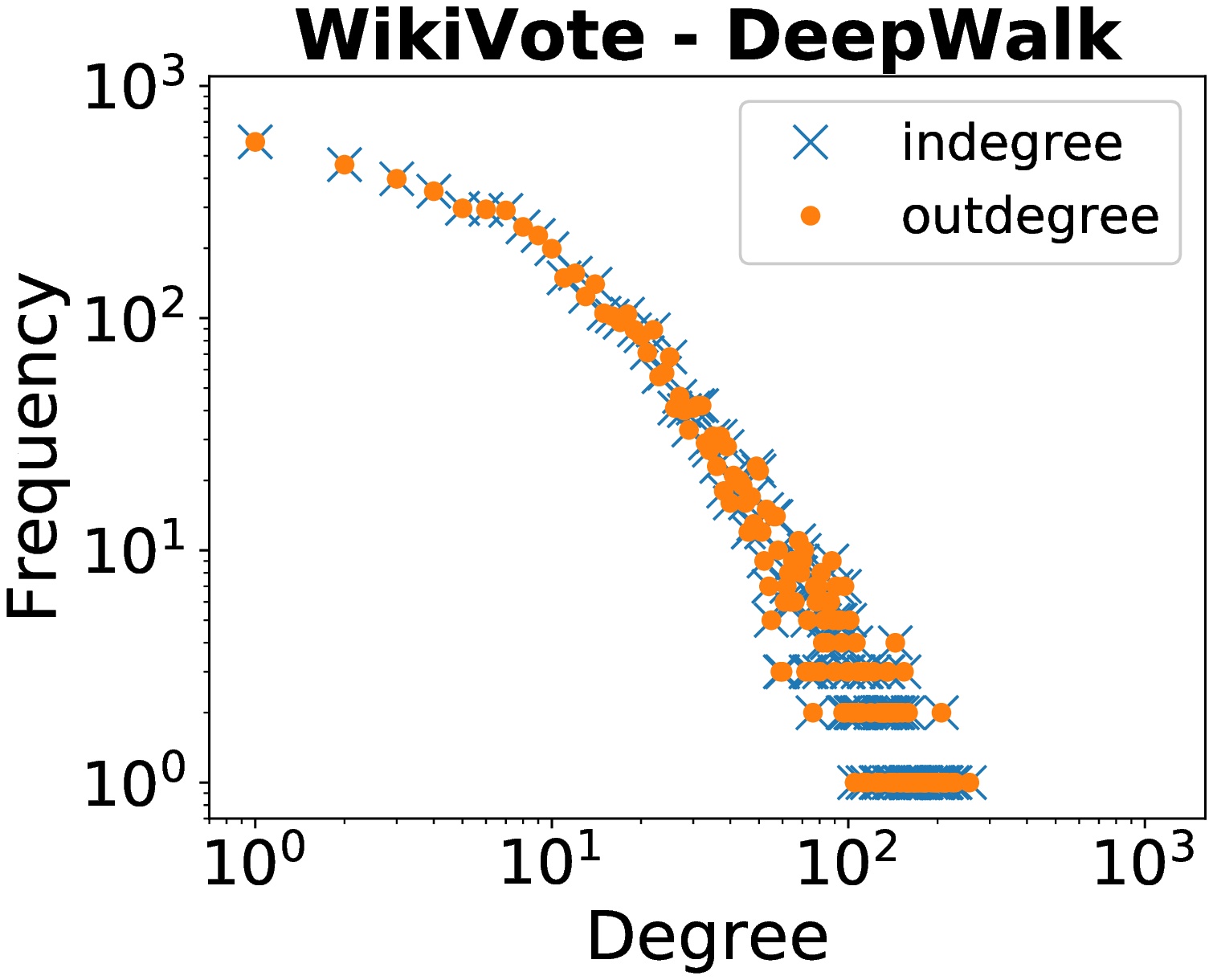}
 \includegraphics[height=28 mm]{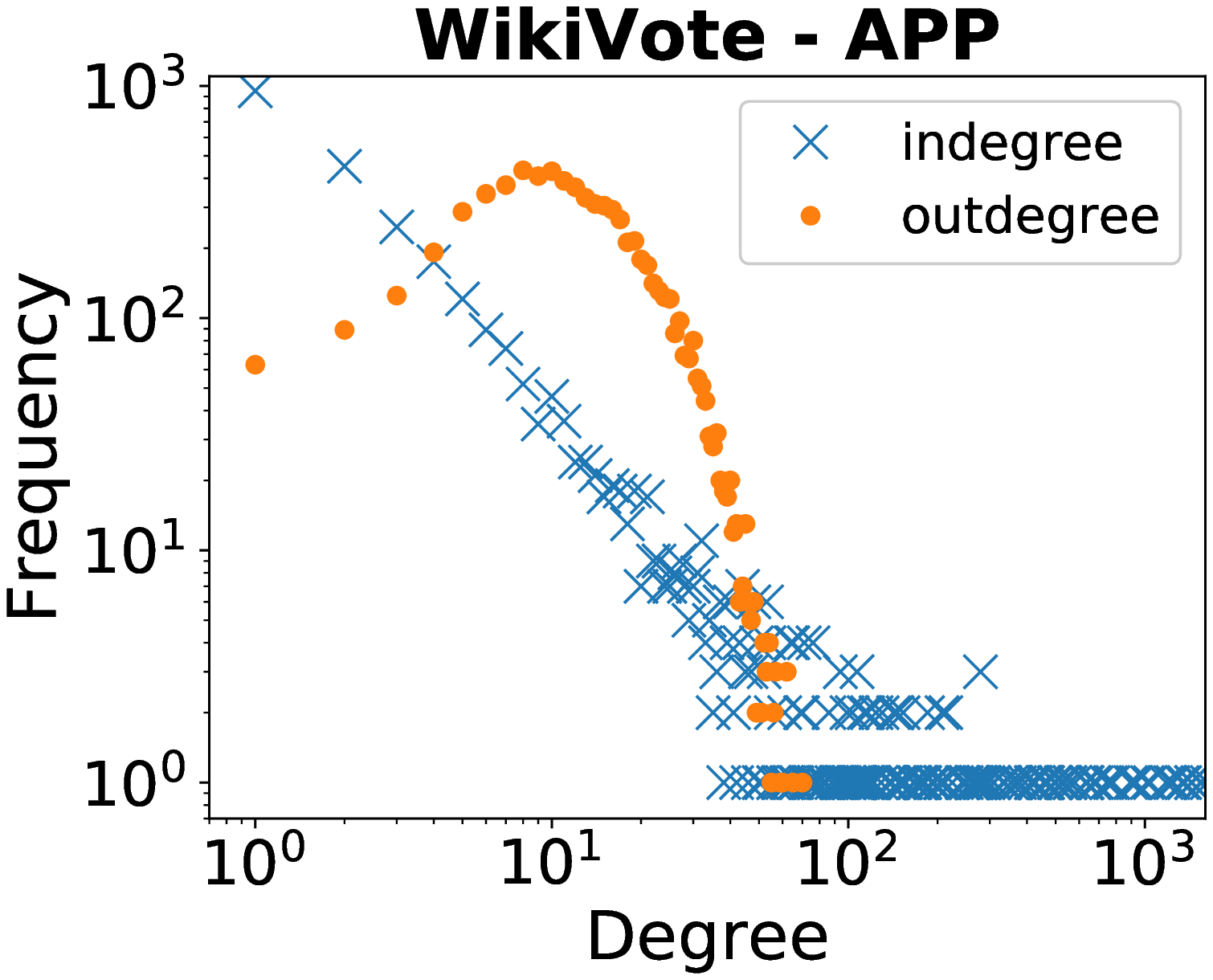}
 \includegraphics[height=28 mm]{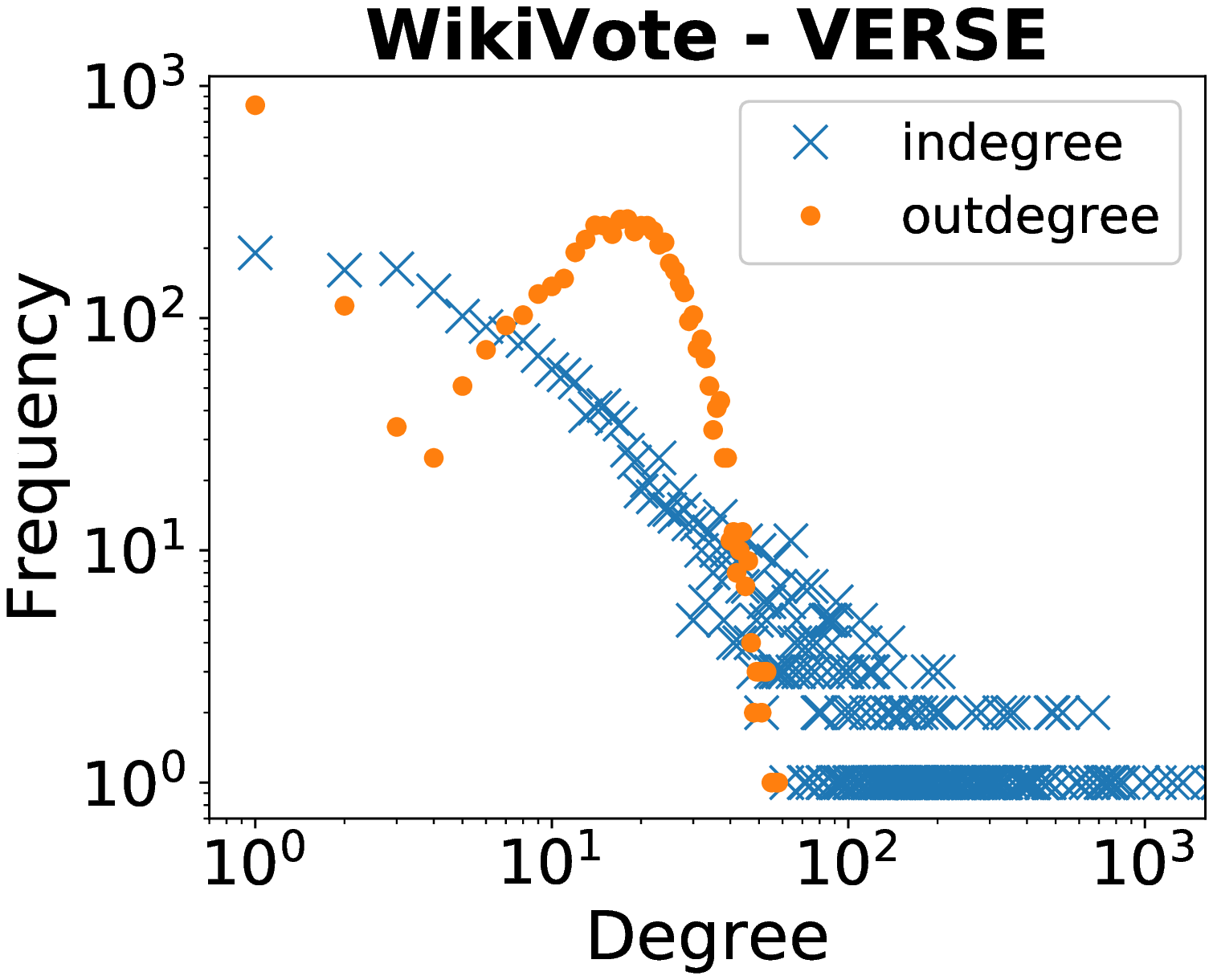}
  \includegraphics[height=28 mm]{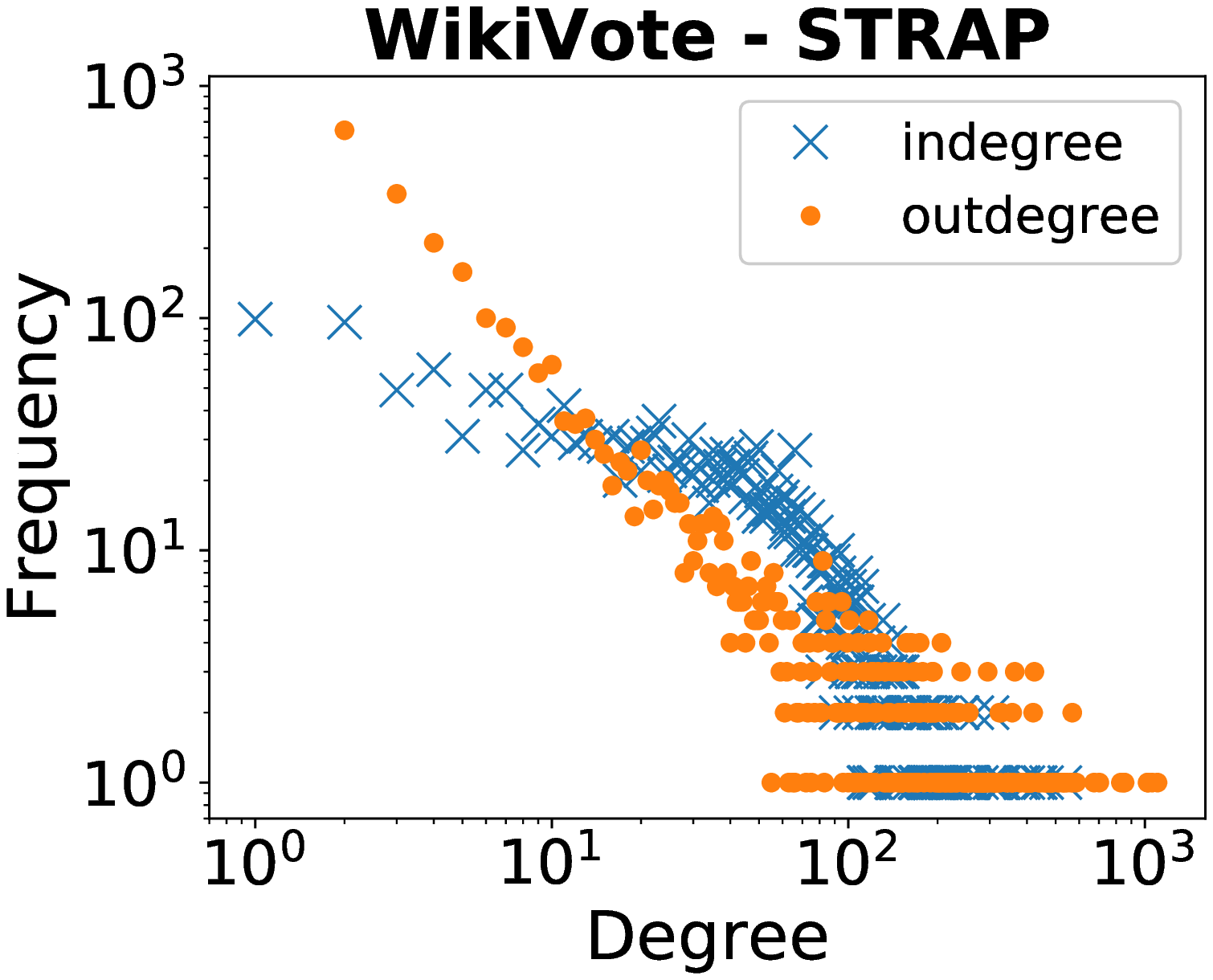}
  \vspace{-3mm}
 \caption{Degree distributions of WikiVote.} \label{fig:degree_wikivote}
\vspace{-3mm}
\end{small}
\end{figure*}

\vspace{-1mm}
\subsection{Motivations and Objectives}
Although a significant amount of progress has been made toward the understanding of graph embedding,  we believe that there are still several fundamental problems that remain unsolved. 
The goal of this paper is to analyze the mechanisms that cause these problems  and to design principles and techniques that solve them. In particular, our objective is to design a graph embedding algorithm with the following desired properties.

\header{\bf Objective 1: preserve both in- and out-degree distributions on directed graphs.} Consider a simple task of reconstructing the directed graph WikiVote with $n=7,115$ nodes and $m=103,689$ edges. We train embedding vectors for each node, rank pairs of nodes according to the inner products of their vectors, remove self-loop, and take the top-$m$  pairs of nodes to reconstruct the graph. Figure~\ref{fig:degree_wikivote} shows the degree distributions of the original graph WikiVote and the reconstructed graphs by several state-of-the-art graph embedding algorithms: DeepWalk~\cite{perozzi2014deepwalk}, APP~\cite{zhou2017scalable}, VERSE~\cite{tsitsulin2018verse}. We exclude the results of some other methods such as HOPE~\cite{ou2016asymmetric} and Node2Vec~\cite{grover2016node2vec},  as the results are similar to those of DeepWalk or VERSE.  We first observe that DeepWalk generates identical in/out-degree distributions. This is because DeepWalk (or Node2Vec) trains a single embedding vector $\s_u$ for each node $u$, and uses the same inner product $\s_u \cdot \s_v = \s_v \cdot \s_u$ to predict for the edge from  $u$ to $v$ and the edge from $v$ to $u$.
%This is as expected, since they only generate a single embedding vector $\s_u$ for each node $u$.
%Therefore, if $\s_u \cdot \s_v$ is  in top-$m$, $\s_v \cdot \s_u$ will also be in top-$m$.
Therefore, DeepWalk (and Node2Vec) is only able to preserve structural information for undirected graphs.

The second observation is that VERSE and APP, the two recent embedding algorithms that are designed for directed graphs,  fail to preserve the out-degree distribution of the original graphs. In particular, the reconstructed out-degree distributions do not follow  power-law distribution: there are no nodes with large out-degrees, and most out-degrees concentrate on 14, the average out-degree of the original graphs. As it turns out, there is a fundamental reason for this phenomenon. Recall that an embedding algorithm determines a proximity measure $P(u,v)$, and tries to train $\s_u\cdot \t_v \sim P(u,v)$. For random walk based proximities such as Personalized PageRank or hitting probability,  the proximities values of node $u$ to any nodes in the graph is normalized, i.e., $\sum_{v\in V}P(u,v) = 1$. Therefore, given a source node $u$,  the number of pairs $(u,v)$ with large proximities $P(u,v)$ (and hence large inner products $\s_u\cdot \t_v$) is actually limited to a very small range. Consequently, these methods are inherently unable to reconstruct nodes with large out-degrees, and the out-degrees of the reconstructed graph will concentrate  on the average out-degree of the original graph. 

We note that the lack of ability to preserve degree distributions will hurt both the effectiveness of the embedding vectors. In particular, these methods are inherently unable to make predictions  for nodes with many or very few out-neighbors. Therefore, our first objective is to study how to modify the proximity measure $P(u,v)$ to preserve both in- and out-degree distributions. 

%{\bf Objective 1: Preserve the degree distribution for directed graphs.} 

% Recall that in the network reconstruction task, we take all pairs $(u,v)$ with similarities $ \s_u\cdot \t_v > \theta$, where $\theta$ is the $m$-th largest similarity, and predicts that there is an edge between $u$ and $v$. However, since $\sum_{v\in V} \s_u\cdot \t_v  = \sum_{v\in V}P(u,v) = 1$, the number of nodes $v$ with $ \s_u\cdot \t_v > \theta$ is less than $1/\theta$. Therefore, the out-degree of $u$ is inherently small. Meanwhile, since $\sum_{v\in V} \s_u\cdot \t_v  =1$, the out-degree distribution in the reconstructed graph will concentrate on the average out-degree of the original graph.

% For example, if we were to predict the links from a hub node $u$ with many out-neighbors, existing method will achieve poor accuracy as they are only able to 

% these reconstructed graphs, the number of vertices with degree 10 are significantly higher than number of vertices with smaller or large degrees. Meanwhile, there are no vertices have large out-degrees. It is easy to see that this deviation of out-degree distribution will hurt the quality of network reconstruction.  

\header{\bf Objective 2: avoid conflicting optimization goals for undirected graphs.} Another more subtle deficiency suffered by existing techniques is the conflicting optimization goals lead by the usage of asymmetric proximities.  More precisely, recall that existing methods such as DeepWalk, Node2Vec and
VERSE  train a single embedding vector  $\s_u$ for each node $u$ on undirected graphs,  such that $\s_u \cdot \s_v \sim P(u,v)$ for some proximity measure $P(u,v)$. Consequently, the algorithms  will train $\s_v \cdot \s_u$ to approximate the proximity $P(v,u)$. We note that the inner product $\s_u \cdot \s_v = \s_v \cdot \s_u$ is commutative, but the proximity $P(u,v)$ generally does not equal to $P(v, u)$, even on undirected graphs.
% We note that the inner products $\s_u \cdot \s_v$ is commutative, that is $\s_u \cdot \s_v = \s_v \cdot \s_u$. However, the the proximity $P(u,v)$ is generally not commutative. 
For example, the probability that a random walk from $u$ visit $v$ in $t$ steps does not equal to the probability that a random walk from $v$ visit $u$ in $t$ steps; the Personalized PageRank of $v$ with respect to $u$ does not equal to the Personalized PageRank of $u$ with respect to $v$. As a consequence, these methods try to train $\s_u \cdot \s_v$ to approximate two conflicting values, which will hurt the quality of the embeddings vector.

On the other hand, HOPE and APP tries to solve this problem by training asymmetric content/context embedding vectors $\s_u$ and $\t_u$ for each node $u$ on undirected graphs, such that $\s_u \cdot \t_v \sim P(u,v)$ and $\s_v \cdot \t_u \sim P(v,u)$. However, this approach introduces another problem: since there may be a substantial difference between $\s_u \cdot \t_v$ and $\s_v \cdot \t_u$, we are unable to determine which to use to predict for edge $(u,v)$ in the task of graph reconstruction or link prediction. Therefore, it is desirable to use symmetric proximities on undirected graphs. At first glance, this requirement  rules out all random walk based proximities. However, as we shall see, we can achieve symmetry by a simple modification.

% \header{\bf 3. Existing methods do not offer tradeoffs between inductive and transductive effectiveness.} Most existing work tries to optimize both 
% inductive (e.g. link prediction) and transductive (e.g. network reconstruction) effectiveness of the embedding vectors. However, we argue that there may not be a solution to achieve best of both world, due to the problem of overfitting. A strong evidence, as shown in our experiments, is that directly applying SVD to the adjacency matrix generally achieves the best transductve accuracy and the worst inductive accuracy. An intuitive explanation of the phenomenon is that the adjacency matrix perfectly preserves the edge information of the original graph, and yet suffers severely from problem of overfitting for any inductive tasks. Therefore, we believe it is desirable and more realistic to provide a tradeoff curve between  inductive and transductive effectiveness.

\header{\bf Objective 3: design factorization method that achieves scalability and non-linearity  simultaneously. }
The general goal of embedding algorithms is to  optimizing both 
inductive (e.g. link prediction) and transductive (e.g. graph reconstruction) effectiveness, and to achieve high scalability on large graphs. Matrix factorization methods usually achieve good transductive effectiveness as they are designed to minimize the reconstruction error of the proximity matrix. However, they suffer from 
scalability problem, since it takes $\Theta(n^2)$ time to compute the proximity matrix. Recently, HOPE and AROPE~\cite{zhang2018arbitrary}  avoid  the $\Theta(n^2)$ computation time by factorizing a sparse matrix that closely related to the proximity matrix. However,  they do not explicitly compute the proximity matrix, and thus do not allow any non-linear operation (such as taking logarithm or softmax) on the proximity matrix. This approach limits their inductive strength due to  the linear nature. In fact, it has been shown in~\cite{zhou2017scalable}
and~\cite{qiu2018network} that skip-gram based algorithms implicitly
factorize the logarithm of certain proximity matrix,  where taking entry-wise
logarithm simulates the effect of the sigmoid function and improves the induction strength of the model. As a result,  it is desirable to design a factorization method that achieves high scalability and allows non-linear operations on the proximity matrix.

\vspace{-2mm}
\subsection{Our Contributions}
To remedy the deficiencies of existing techniques, this paper presents an in-depth study of graph embedding techniques on both directed and undirected graphs. 
First, given a normalized proximity measure $P(u,v)$, we propose that instead of training $\s_u \cdot \t_v \sim P(u,v)$, we should train $\s_u \cdot \t_v $ to approximate the  {\em transpose proximity} $P(u,v) + P^T(v,u) $, where $P^T(v,u) $ is the proximity of $u$ with respect to $v$ in the transpose graph $G^T$. Here $G^T$ is obtained by reverting the edge direction of the original graph $G$. We show  that by this simple modification,  we solve the distortion of out-degree distributions and  the conflicting optimization goals simultaneously.

Based on the concept of transpose proximity,  we propose {\it \strap} (graph embedding via Sparse TRAnspose Proximities), an embedding algorithm that provides both high predictive strength and scalability. See Figure~\ref{fig:degree_wikivote} for the reconstructed degree distributions of WikiVote by \strap.
We use {\em Personalized PageRank (PPR)}  as the normalized proximity measure $P(u,v)$ to demonstrate the superiority of transpose proximity. 
To avoid the $\Theta(n^2)$ barrier of computing pair-wise PPR, we employ the {\em backward push} algorithm~\cite{lofgren2015personalized} that computes approximate pair-wise PPR values with additive error $\e$ in $O(m/\e)$ time.  Unlike HOPE or AROPE, we explicitly derive the proximity matrix $P$, a sparse matrix that consists of at most $O(n/\e)$ non-zero entries. The sparsity enables us to impose non-linear functions such as entry-wise logarithm to improve the predictive strength, as well as to use sparse SVD algorithm to efficiently decompose $P$ into the embedding vectors. We experimentally evaluate \strap on a variety of benchmark datasets, and our results demonstrate that \strap outperforms state-of-the-art methods for both transductive and inductive tasks.

% \begin{figure*}[!t]
% \begin{small}
%  \centering
%    %%\vspace{-2mm}
% %    \begin{footnotesize}
%      \begin{tabular}{cccc}
%         \hspace{-8mm} \includegraphics[height=33mm]{./Figs/exact.eps}
%    &
%         \hspace{-8mm} \includegraphics[height=33mm]{./Figs/approximate.eps}&
%         \hspace{-8mm} \includegraphics[height=33mm]{./Figs/preprocess.eps} &
%         \hspace{-4mm} \includegraphics[height=33mm]{./Figs/space.eps}
%    %\vspace{-1mm} \\
%        \hspace{-8mm} (a) {Query time for exact top-$k$ queries}  &
%        \hspace{-4mm} (b) {Query time for approximate top-$k$ queries} &
%        \hspace{-8mm} (c) {Preprocessing time} &
%        \hspace{-8mm} (d) {Space usage} \\
%  \end{tabular}
% %\vspace{-3mm}
%  \caption{Experiments} \label{fig:experiements}
% %\vspace{-3mm}
% \end{small}
% \end{figure*}

% \header
% {\bf Technical difficulties.}
% Question1: when should we perform backward push, and on which nodes? Intuitively, backward push should be performed
% Question 2: How far should the forward/backward push go, if we care about the top-$k$ results, rather than the $\e$ relative error guarantee? Question 3: can we return the exact top-$k$ PPR nodes with high confidence, if extra query time is allowed? In many cases, we only care about the precision of the top-$k$ results, not the actually PPR values.
%%% Local Variables:
%%% mode: latex
%%% TeX-master: "paper"
%%% End:

%!TEX root=../ssppr_kdd17_newformat.tex

\vspace{-2mm}
\section{Related work} \label{sec:related}

%\subsection{Graph Embedding Techniques}

In general,  there are broadly two categories of approaches: methods
which use random walks to learn the embeddings, and methods
which use matrix factorization to directly derive the embeddings. We briefly review some of the relevant works in each
category and refer readers to~\cite{cui2018survey,goyal2018graph,zhang2018network} for
comprehensive surveys.

\header{\bf Factorization methods.} A natural idea for preserving
the high-order proximities is to perform explicit matrix factorization
on the proximity matrices. Early efforts include
LLE~\cite{roweis2000nonlinear}, Laplacian
Eigenmaps~\cite{belkin2002laplacian} and
Graph Factorization~\cite{ahmed2013distributed}.
GRAREP~\cite{cao2015grarep} performs SVD on the $k$-th order
transition matrix, and GEM-D~\cite{chen2017fast} gives a unified approach to compute
and factorize the proximity matrix for various proximity measures. \cite{qiu2018network} shows that existing random walk
based methods are equivalent to factorizing high order proximity
matrices.  However, computing the proximity
matrix for the above methods 
still takes $\Theta(n^2)$ time, and hence are inherently not
scalable. HOPE~\cite{ou2016asymmetric} avoid the $\Theta(n^2)$ time  by performing 
a special version of SVD on proximity matrix of form $M_g^{-1}M_\ell$,
where both $M_g$ and $M_\ell$ are
sparse.  Recently, \cite{zhang2018arbitrary} proposes AROPE, a general
framework for preserving arbitrary-order proximities that includes HOPE as its
special case. However, HOPE and AROPE do not compute the explicit
proximity matrix, and thus are unable to support any non-linear operation on the proximity matrix, which
limits their inductive strength due to  the linear
nature.

\header{\bf Random walk methods.} Random walks have been used to
approximate many different proximities such as Personalized PageRank~\cite{page1999pagerank},
Heat Kernel PageRank~\cite{kloster2014heat} and
SimRank~\cite{JW02}. In the line of graph embedding research, 
DeepWalk~\cite{perozzi2014deepwalk}  first proposed to train embedding vectors by feeding
truncated random walks to the SkipGram
model~\cite{mikolov2013distributed}. The model is optimized by
Stochastic Gradient Descent (SGD).  LINE~\cite{tang2015line} focuses
on one-hop neighbor proximity, which essentially equals to random
walks with step at most 1. Node2Vec~\cite{grover2016node2vec} proposes to replace the truncated
random walks with a higher order random walks that exploit both DFS
and BFS nature of the graph. Recently, Verse~\cite{tsitsulin2018verse}
and APP~\cite{zhou2017scalable} propose to train embedding vectors
using Personalized PageRank, where the positive samples can be
efficiently obtained by simulating $\alpha$-discounted random walks. Random
walk methods are scalable than some of the factorization methods and
generally achieve higher inductive effectiveness. However, they only
allow normalized proximity measure, which, as we shall see, leads to
the distortion of out-degree distributions on directed
graphs.

\header{\bf Other related work.}
The growing research on deep learning has led to a deluge of deep
neural networks based methods applied to graphs
\cite{wang2016structural,cao2016deep,niepert2016learning,bhuiyan2018representing}. Recently,
Graph Convolutional Network (GCN)~\cite{kipf2016semi} and its variants
have drawn increasing research attention. 
There are also various graph embeddings designed for specific graphs, such as
signed graphs~\cite{kim2018side,yuan2017sne}, dynamic
graphs~\cite{zhou2018dynamic,ma2018depthlgp} and heterogeneous
networks~\cite{chang2015heterogeneous,dong2017metapath2vec}. In this paper, we focus on the most fundamental
case  where only the static network is available.  

% \subsection{Personalized PageRank.}
% Given a source node $s$, a target node $t$ on directed graph $G =
% (V,E)$, the {\em Personalized PageRank (PPR)} of $t$ with respected to $s$
% measures the importance of $t$ in the view of $s$~\cite{page1999pagerank}. 
% More precisely, we
% define a
% {\em $\alpha$-discounted random walk} %(or more precisely, random walk with restart \cite{FujiwaraNYSO12})
% from $s$ to be a traversal of $G$ that starts from $s$ and, at each step,
% either (i) terminates at the current node with $\alpha$ probability,
% or (ii) proceeds to a randomly selected out-neighbor of the current
% node. For any node $v \in V$, the  (PPR)
% $\pi(s,v)$ of $v$ with respect to $s$ is the probability that an
% $\alpha$-discounted random walk from $s$ terminates at $v$.  PPR has been an important building block of numerous web search and social network applications, such as Twitter's {\em Who-To-Follow} \cite{gupta2013wtf}, LinkedIn's connection recommendation \cite{LinkedIn}, and Pinterest's {\em Related Pins} \cite{LiuRSKMZLJ17}. 

%%% Local Variables:
%%% mode: latex
%%% TeX-master: "paper"
%%% End:

\vspace{0mm}
\section{strap  algorithm}
In this section, we present \strap, a scalable graph embedding
algorithm that achieves all three
objectives. Table~\ref{tbl:notations} summaries the frequently used
notations used in this paper. We first show that a
unified approach, {\em transpose proximity,} achieves {\bf Objective
  1} and {\bf Objective 2} simultaneously.
\begin{table}[!t]
\centering
\begin{small}
\renewcommand{\arraystretch}{1.2}
%\vspace{0mm}
\caption{Frequently used notations.} \label{tbl:notations}
\vspace{-3mm}
\begin{tabular}{|p{0.8in}|p{2.4in}|}
    \hline
    {\bf Notation} &  {\bf Description}\\
    \hline
    $G$=$(V,E)$   & The input graph $G$ with node set $V$ and edge set $E$\\
    \hline
    $n,m$   & The number of nodes and edges in $G$, respectively\\
    \hline
    $\outN(v)$, $\inN(v)$  & The set of  out- and  in-neighbors of node $v$ \\
     \hline
    $d_{out}(v)$, $d_{in}(v)$  & The out-degree and in-degree of node
                                 $v$ \\
   \hline
    $\s_u, \t_u$   & The content/context embedding vectors of
                     $u$\\
     \hline
    $\ppr(u, v)$   & The Personalized PageRank  of $v$ with respect to $u$\\
    \hline
    $\alpha$   & The decay factor\\
    \hline
    $\rb(u,v)$, $\pib(u,v)$ & The reserve and residue of $v$ from $u$
                              in backward push \\
    \hline

\end{tabular}
\vspace{-5mm}
\end{small}
\end{table}

\vspace{-2mm}
\subsection{Transpose Proximity}
Suppose the goal of a graph embedding algorithm is to train
content/context embedding vectors $\s_u$ and $\t_u$ for each node $u \in
V $, such that $\s_u\cdot \t_v\sim P(u, v) $ for a predetermined
proximity measure $P(u,v)$. We assume the
proximities of any node with respect to a given node $u$ is
normalized, i.e., $\sum_{v\in V} P(u, v) =1$. This assumption holds
for any random walk based proximities (e.g. Personalized PageRank,
hitting probability etc), and thus is well-recognized by various graph
embedding algorithms.  Let
$G^T$ denote the transpose graph of $G$, that is, there is an edge
from node $v$ to node $u$ in $G^T$ if and only if there is an edge
from node $u$ to node $v$.

The key insight of transpose proximity is that instead of optimizing $\s_u \cdot \t_v \sim P(u,v)$,
we should optimize $\s_u \cdot \t_v \sim P(u,v) + P^T(v, u)$, where
$P^T(v, u)$ is the proximity of $u$ with respect to $v$ in the
transpose graph $G^T$. We will show
that 1) the transpose proximity
preserves both in- and out-degree distributions for directed graphs
and 2) the transpose proximity avoid conflicting optimization goals on
undirected graphs.

\header{\bf In-degree distribution and proximities.} To show that the transpose proximity
preserves both in- and out-degree distributions for directed graphs,
we first establish the connection between the in-degree distributions
and normalized proximities.
We observe from Figure~\ref{fig:degree_wikivote} that although existing methods do not preserve
the out-degree distribution,
they  generate power-law-shaped in-degree distributions that are
similar to the one of the original
graph. An intuitive explanation is that although the proximity sum
$\sum_{v\in V} P(u, v)$ from a source node $u$ is normalized to $1$, the
proximity sum 
$\sum_{u\in V} P(u, v)$ to a target  node $v$ is not
normalized. In fact, $\sum_{u\in V} P(u, v)$ reflects the number of
nodes that are
similar to $v$ and thus is a good approximation
to the indegree of $v$.  For example, VERSE, APP and HOPE use
Personalized PageRank of $v$ with respect to $u$
as the normalized proximity measure $P(u, v)$. For this particular
proximity measure, we have  $\sum_{u\in V}
P(u, v) = n\cdot \mathrm{PR}(v) $~\cite{page1999pagerank}, where $n$ is the number of nodes in the graph,
and $\mathrm{PR}(v)$ is the PageRank of $v$. It is shown
in~\cite{BahmaniCG10,lofgren2015personalized,wei2018topppr} that on
scale-free networks, 
PageRank and in-degrees follow the same power-law distribution, which
implies that $\sum_{u\in V}
P(u, v) = n\cdot \mathrm{PR} (v)  \sim \din(v)$.  Consequently, we claim the
reason that 
existing method preserves the in-degree distribution is that
they employ
normalized proximity $P(u,v)$ that satisfies $  \sum_{u\in V} P(u,v) \sim \din(v).$
% \begin{equation}
%   \label{eqn:in_degree}
%   \sum_{u\in V} P(u,v) \sim \din(v).
%   \end{equation}

\header{\bf In-/out-degree distributions and transpose proximities.} 
With the insight that $  \sum_{u\in V} P(u,v) \sim \din(v)$, % the help of equation~\eqref{eqn:in_degree}
we now show
that transpose proximity $P(u,v) + P^T(v, u)$ preserves both the in- and
out-degree distribution. Consider the summation of transpose
proximities to a target node $v$,  we have $  \sum_{u \in V}\left(P(u, v) + P^T(v,u) \right) = \sum_{u \in V}P(u,
  v) + \sum_{u \in V}P^T(v,u)= \sum_{u \in V}P(u,v) + 1 \sim \din(v).$
%   \vspace{-2mm}
% \begin{align*}
%   \sum_{u \in V}\left(P(u, v) + P^T(v,u) \right) &= \sum_{u \in V}P(u,
%   v) + \sum_{u \in V}P^T(v,u) \\
%                                                  &= \sum_{u \in V}P(u,v) + 1 \sim \din(v).
% \end{align*}
Note that here we use the fact that $ \sum_{u \in
  V}P(u,v) \sim \din(v)$ and ignore the plus one since it does not
change the relative order of the proximity summations.
On the other hand, let  $\din^T(u)$ denote the in-degree of $u$ in the transpose graph
$G^T$.  Consider the summation of  transpose proximities from a source
node $u$ and we have $\sum_{v \in V}\left(P(u, v) + P^T(v,u) \right) = \sum_{v \in V}P(u,
                                                 v) + \sum_{v\in V}P^T(v,u) 
  =1+ \sum_{v\in V}P^T(v,u)  \sim
    \din^T(u).$
% \vspace{-2mm}
% \begin{align*}
% \sum_{v \in V}\left(P(u, v) + P^T(v,u) \right) &= \sum_{v \in V}P(u,
%                                                  v) + \sum_{v\in V}P^T(v,u) \\
%   &=1+ \sum_{v\in V}P^T(v,u)  \sim
%     \din^T(u),
% \end{align*}
Here we use the fact $ \sum_{u
  \in V}P^T(v,u) \sim \din^T(u)$ in the transpose graph
$G^T$. We observe that $ \din^T(u)$, the in-degree of $u$ in the transpose graph
$G^T$, equals to $\dout(u)$, the out-degree of $u$ in the original graph
$G$. It follows that the summation $\sum_{v\in V} \left(P(u, v) + P^T(v,u) \right) \sim \dout(u)$.
% \begin{equation}
%   \label{eqn:transpose_in_degree}
%   \sum_{v\in V} \left(P(u, v) + P^T(v,u) \right) \sim \dout(u).
% \end{equation}
Therefore the summation  of transpose proximities to a target node $v$ approximates the in-degree of $v$, while
the summation of transpose proximities from a source node $u$
approximates the out-dgree of $u$. As 
a consequence, by employing the transpose proximity $P(u,v) + P^T(v,
u)$, we preserve both the in- and
out-degree distribution for directed graphs.

\header{\bf Transpose proximity on undirected graphs.} 
Another advantage of transpose proximity is that it
automatically avoids the
conflicting optimization goals on undirected graphs. In particular,
note that for undirected graphs, the transpose graph $G^T$ is
identical to the original graph $G$, and thus
$P^T(v, u) $ equals to $P(v, u)$. Therefore, the
transpose proximity becomes $P(u,v)
+ P(v,u) $, which is a symmetric similarity measure for any
proximity measure $P(u,v)$.  Consequently, we train $\s_u \cdot \s_v = \s_v \cdot \s_u$
to approximate the same proximity $P(u,v)
+ P(v,u) = P(v,u)
+ P(u,v)$, and thus avoiding the conflicting optimization goals
suffered by existing techniques. % Furthermore, by using proximity
% measure $P(u,v) + P(v,u) $ for factorization methods, the proximity
% matrix $P$ is a symmetric matrix, and thus we can perform
% eigen-decomposition to derive a single embedding vector $\s_u$ for
% each node $u$. 

\vspace{-2mm}
\subsection{Sparse Personalized PageRank}
Although the concept of transpose proximity works for any normalized
proximity measure $P(u,v)$, in this paper we focus on $P(u,v) = \ppr(u,v)$, the  {\em Personalized
  PageRank (PPR)}~\cite{page1999pagerank} of node $v$ with respect to
node $u$. Given a source node $u$, a target node $v$ on directed graph $G =
(V,E)$, $\ppr(u,v)$
measures the importance of $v$ in the view of $u$. 
More precisely, we define an
{\em $\alpha$-discounted random walk} %(or more precisely, random walk with restart \cite{FujiwaraNYSO12})
from $u$ to be a traversal of $G$ that starts from $u$ and, at each step,
either 1) terminates at the current node with $\alpha$ probability,
or 2) proceeds to a randomly selected out-neighbor of the current
node. For any node $v \in V$, 
$\ppr(u,v)$ of $v$ with respect to $u$ is the probability that an
$\alpha$-discounted random walk from $u$ terminates at $v$.
We choose PPR mainly because it has been widely used in graph embedding 
algorithms~\cite{ou2016asymmetric,tsitsulin2018verse,zhou2017scalable}.
Moreover, as we stated before, the summation of
$\sum_{u\in V} \ppr(u,v) = n\mathrm{PR}(v)$ equals  the PageRank of
$v$, and PageRank and in-degrees follow the same power-law
distribution. Therefore, by employing transpose proximity matrix $P$ with
$P_{uv}=\ppr(u,v) + \ppr^T(v,u)$, the resulting embeddings will  preserve both the in- and
out-degree distributions of the graphs.

% \header{\bf Objective 3: Provide tradeoffs between inductive and
%   transductive effectiveness.} By using SPPR, we can provide tradeoffs
% between inductive and transductive effectiveness by manipulating the
% decay factor $\alpha$. Intuitively, for $alpha$ close to $1$, the
% probability mass will focus on one-hop neighbors. Therefore, the
% proximity matrix $P$ will preserve the adjacency information. As we move
% $alpha$  to $0$, the information of multi-hop neighbors will be added
% to SPPR, which increases the inductive effectiveness but decrease the
% ability of preserving the original graph.

% 1)
% Personalized PageRank is a random walk based proximity measure. It has
% been shown that $PPR(u,v)$ is the probability that a
% $\alpha$-discounted random walk from $u$ terminates at $v$. Here an
% $\alpha$-discounted random walk is a random walk that at each step,
% proceeds to a random out-neighbour with probability $1-\alpha$ and
% terminates at current vertice with probability $alpha$; 2)
% Personalized PageRank can be approximated using local search
% algorithms; 3) Several existing graph embedding (Verse, APP) techniques uses
% Personalized PageRank as the proximity measure. As we emphasized in
% Section, our goal is not to select a best proximity measure, but
% rather to show the benefit of introducing transpose proximities. 

However, directly computing PPR for any node pair $(u,v)$
takes at least $\Theta(n^2)$ time and memory usage.  To make thing worse, it takes $O(n^3)$ time to decompose a $n\times
n$ dense proximity matrix. Therefore, it is infeasible to compute exact
PPR for all node pairs on
large graphs. HOPE~\cite{ou2016asymmetric}  proposes to decompose the PPR matrix $\ppr =
M_g^{-1}M_\ell$ by performing a generalized SVD on sparse matrices
$M_g$ and $M_\ell$, where $M_g= I -(1-\alpha) D^{-1}A$, $M_\ell =
\alpha I$, $D$ is the diagonal degree matrix and $A$ is the
adjacency matrix. However, this approach does not explicitly compute
the PPR matrix and thus does not support decomposition of the transpose
proximity matrix $P$ where $P_{uv} = \ppr(u,v) + \ppr^T(v,u)$. Furthermore, it does not allow
non-linear operations before the
decomposition, which is crucial for achieving satisfying predictive strength~\cite{qiu2018network}.

To explicitly compute PPR values for all pairs of nodes in the graph
efficiently, we will use an approximate version of PPR
called {\em Sparse Personalized PageRank (SPPR)}. Given an error
parameter $\e \in (0,1)$  and two nodes $u$
and $v$ in the graph, $SPPR(u,v)$ is a real value that  satisfies:1) 
  $|SPPR(u,v) - PPR(u,v)| \le \e $, for any $u, v \in V$; 2) For a
fixed node $u$, there are at most $2/\e$ nodes $v$ with non-zero $SPPR(u,v)$.
Note that the first condition guarantees that the sparse
Personalized PageRank approximates the original Personalized PageRank
with  precision $\e$. This relaxation allows us to compute SPPR in
time linear to the edge number $m$. On the other hand, the second condition ensures that
the proximity matrix is  sparse, which is crucial for efficient
 matrix decomposition. Combining the idea of transpose proximity,
our final proximity measure is defined as
$P_{uv} = SPPR(u,v) + SPPR^T(v,u)$. 

% For example, if we allow an
% additive error of at most $0.001$, the non-zero entries of each row in
% the proximity matrix is at most $1000$. For comparison, approximating
% $PPR(u,v)$ with additive error $0.001$ requires at least $1000^2 =
% 10^6$ random walks. Finally, 

\vspace{-2mm}
\subsection{Computing  SPPR with Backward Push}
\vspace{-3mm}
  \begin{algorithm}[h]
\begin{small}
\caption{Backward Push~\cite{lofgren2015personalized}} \label{alg:bp}
\BlankLine
\KwIn{Graph $G$, target node $v$, decay factor $\a$, threshold $\brmax$ }
\KwOut{Backward residue $\rb(x,v)$ and reserve $\pib(x,v)$ for all $x\in V$}
$\rb(v, v) \gets 1$ and $\rb(x,v) \gets 0$ for all $x \ne v$;
$\pib(x,v) \gets 0$ for all $x$\;
\While {$\exists x$ such that $\rb(x,v)> \brmax$}
{
    \For {each $y \in \inN(x)$}
    {
       $\rb(y, v) \gets \rb(y, v) + (1-\alpha) \cdot \frac{ \rb(x, v)}{d_{out}(y)}$
    }
    $\pib(x,v) \gets \pib(x,v) + \alpha  \cdot \rb(x, v)$\;
    $\rb(x, v) \gets 0$\;
  }
\end{small}
\end{algorithm}
\vspace{-4mm}
\header{\bf Backward Push.} We employ a local search algorithm called {\em Backward
  push}~\cite{lofgren2015personalized} to compute SPPR for any node pair $(u,
v)$ in $O(m / \e)$ time. Given a {\em destination} node $v$, the
backward push algorithm employs a
traversal from $v$ to compute an approximation of $v$'s PPR value $\ppr(u, v)$ with respect
to any other node $u$.  We sketch the algorithm in Algorithm~\ref{alg:bp} for
completeness.   
The algorithm starts by assigning a residue $\rb(x, v)$
and reserve $\pib(x, v) = 0$ to each node $x$, and setting $\rb(v, v)
= 1$ and $\rb(x, v) = 0$ for any $x \ne v$ (Lines 1-2). Subsequently,
it traverses from $v$, following the incoming edges of each node. For
any node $x$ that it visits, it checks if $x$'s residue $\rb(x, v)$ is
larger than a given threshold $\brmax$. If so, then it increases $x$'s
reserve by $\alpha \times \rb(x, v)$ and, for each in-neighbor $y$ of
$x$, increases the residue of $y$ by $(1-\alpha)\cdot \frac{\rb(x,
  v)}{d_{out}(y)}$ (Lines 4-5). After that, it reset $x$'s residue
$\rb(x, v)$ to $0$ (Line 6). The following Lemma is proven by Lofgren et al.\
\cite{lofgren2015personalized}:

\begin{lemma}[\cite{lofgren2015personalized}]
  \label{lem:backward}
The amortized time of backward push over all possible target nodes $v\in V$ is $O\left(\frac{m}{n\cdot
    \brmax}\right)$.  When the algorithm terminates, it provides a
reserve $\pib(u,v)$ for any node $u$, such that 
$$\ppr(u, v) - \brmax \le \pib(u, v)  \le  \ppr(u, v).$$
  \end{lemma}

\vspace{-3mm}
\begin{algorithm}[h]
\begin{small}
\caption{STRAP} \label{alg:strap}
\BlankLine
\KwIn{Graph $G$, dimension $d$, error parameter $\e$, decay factor $\alpha$ }
\KwOut{Embedding vectors $\s_u$ and $\t_u$ for each $u\in V$}
Initialize sparse proximity matrix $P_{n \times n}\gets 0$\;
\For {each node $v \in V$}
{
  $Backward Push(v, \alpha,
  \e/2, G)$\;
  \For{each node $u$ with reserve $ \pib(u,v) \ge \e/2$}
  {
      $P_{uv} \gets  \pib(u,v)$\;
  }
}

\For {each node $u \in V$}
{
  $Backward Push(u, \alpha,
  \e/2, G^T)$\;
  \For{each node $v$ with reserve $ \pib^T(v,u) \ge \e/2$}
  {
     $P_{uv} \gets P_{uv}  +  \pib^T(v,u)$\;
  }
}
Set sparse matrix $P \gets \log\left({2\over \e}\cdot  P \right)$ for
non-zero entries\;
$[U, \Sigma, V] \gets \mathrm{Randomized\, SVD}(P, d)$\;
\Return $U\sqrt{\Sigma}$ and $V\sqrt{\Sigma}$ as embedding vectors\;
\end{small}
\end{algorithm}
\vspace{-3mm}

\header{\bf \strap algorithm.}  Algorithm~\ref{alg:strap} illustrates the pseudocode of \strap. Recall that our goal is to compute the proximity matrix $\P$, which
  consists of entries $\P_{uv} = \sppr(u,v) + \sppr^T(v,u)$.
  To compute $\sppr(u,v)$ for any node pair $u,v \in V$, we perform backward
push on each target node $v \in V$ with $\brmax = \e/2$ (Lines
2-3). This will give
us a list of node-reserve pairs $(u, \pib(u,v))$. For each node $u$ with reserve $ \pib(u,v) > \e/2$,  we update the
proximity matrix $\P$ by $P_{uv} \gets  \pib(u,v)$  (Lines 4-5). We
claim that $P_{uv} = \sppr(u, v)$ at this time point. We then perform the same process on
each node $u$ in $G^T$ to
compute $\sppr^T(v,u)$ (Lines 6-9). The
only difference is that for  each node $v$ with reserve
$\pib^T(v,u) > \e/2$, we increment $P_{uv}$ (instead of $P_{vu}$) by
$\pib^T(v,u)$ (Line 9). We have the
following Lemma that shows that $P$ is a sparse matrix that
approximates the transpose proximity for any node pair $(u,v)$:

\begin{lemma}
  \label{lem:proximity_matrix}
  The proximity matrix $P$ satisfies 1) There are at most $4n/\e$
  non-zero entries in $P$; 2) For any $u,v \in V$, we have
    $$\ppr(u,v) + \ppr^T(v,u) -2\e \le P_{uv} \le \ppr(u,v) +
      \ppr^T(v,u).$$
  \end{lemma}
\begin{proof}
    As target node $v$ iterates over all possible
nodes in $V$, we ensure that for any node pair $(u,v)$, $P_{uv} =  \pib(u,v)$ if $\pib(u,v) \ge
\e/2$ and  $P_{uv} =  0$ otherwise. We will show that $P_{uv} $ is a
valid SPPR.  By the property of backward push, we
have
$$\ppr(u, v) - \e/2 = \ppr(u, v) - \brmax \le \pib(u, v)  \le  \ppr(u,
v)$$
for any $u, v \in V$. Since we only take $P_{uv} = \pib(u,v)$ with $ \pib(u,v) \ge
\e/2$, it follows that
$P_{uv}  \le \pib(u,v) \le \ppr(u,v),$
and similarly
$$P_{uv}  \ge \pib(u,v) - \e/2 \ge \ppr(u,v) - \e/2 - \e/2 =
\ppr(u,v) - \e.$$
Therefore, we have $\ppr(u,v) - \e \le P_{uv} \le
\ppr(u,v)$, and the first condition of $\sppr$ is satisfied. To see
that $\P_{uv}$ satisfies the
sparsity condition, note that we only take $\P_{uv} = \pib(u,v)$ with $ \pib(u,v) \ge
\e/2$, which means $\P_{uv}  \ge \e/2$ for any $u, v\in V$. Since the summation $\sum_{v\in V}\P_{uv} $ from a source
node $u$ satisfies
$\sum_{v\in V} P_{uv}\le \sum_{v\in V} \pib(u, v) \le  \sum_{v\in V}
\ppr(u,v) = 1,$
it follows that there are at most $2/\e$ non-zero $\P_{uv}$ entries
for a given source node $u$. Consequently, each row of $\P$ contains
no more than $2/\e$ non-zero
entries, adding to a  total of $2n/\e$
non-zero entries.  Let $P'_{uv}$ be the increment to $P_{uv}$ in line
9. By a
similar argument, we have $\ppr^T(v,u) - \e \le P'_{uv} \le
\ppr^T(v,u)$, and thus
$$\ppr(u,v) + \ppr^T(v,u) -2\e \le P_{uv} \le \ppr(u,v) + \ppr^T(v,u).$$
Finally, there are at most $2/\e$ non-zero $P'_{uv}$'s
for a given target node $v$, which means the backward pushes on $G^T$ adds at most $2n/\e$
non-zero entries, resulting at most $4n/\e$ non-zero entries in
the final proximity matrix $P$. Note that despite its sparsity,  the final proximity matrix
is not row-sparse or column-sparse, and thus is able to capture nodes
with large in- or out-degrees.
 \end{proof}
 
% To compute the transpose proximity $\sppr(u,v) + \sppr^T(v,u)$, we continue to perform backward push for each
% node $u$ in the transpose graph $G^T$.  For each tuple $(v,
% u, \pib^T(v,u))$ returned by the backward push algorithm, we update the
% proximity matrix $P$ by $P_{uv} \gets P_{uv} +  \pib^T(v,u)$ if
% $\pib(v,u) \ge \e/2$.  

\header{\bf Achieving non-linearity.}
After obtaining the sparse proximity matrix $P$, we  perform 
logarithm to each non-zero entry in $P$ (Line 10). It has been shown in~\cite{zhou2017scalable}
and~\cite{qiu2018network} that skip-gram based algorithms implicitly
factorize the logarithm of certain proximity matrix,  where taking entry-wise
logarithm simulates the effect of the softmax function. We also
multiply the proximity by $2/\e$ inside the logarithm, such that all
entries of $P$ remains non-negative after we take entry-wise logarithm.

\subsection{Sparse Randomized SVD}
We perform truncated singular value
decomposition (tSVD) to decompose the proximity matrix $P$ into three matrices
$U$, $\Sigma$, and $V$, where $U$ and $V$ are $n\times d$ unitary
matrices, and $\Sigma$ is a diagonal matrix. It is folklore that the
reconstruction matrix $U\Sigma V^T$ is the best-$d$ approximation
to matrix $P$, i.e.
$$\|P -  U\Sigma V^T\|_F = \min_{rank(B)
  \le d} \|P -  B\|_F ,$$
where $\|A\|_F$ denote the Fobenius norm of matrix $A$, 
After the
decomposition, we can return $U\sqrt{\Sigma}$ and
$V\sqrt{\Sigma}$ as the content/context embedding vectors (Lines 11-12).

However, applying traditional truncated SVD to a $n \times n $ matrix requires
$O(n^2d)$ time, which is not feasible when $n$ is large. To reduce this
time complexity, we make use of the fact  that $P$ is a sparse matrix with at most $4n/\e$
non-zero entries.  In particular, we use
{\em Sparse Subspace Embedding}~\cite{clarkson2017low}, which allows us to decompose $P$ into three matrices
$U'$, $\Sigma'$ and $V'$, where $U'$ and $V'$ are $n\times d$ unitary
matrices, and $\Sigma' $ is a diagonal matrix, such that
$$ \|P -  U'\Sigma' V'^T \|_F\le (1+\delta) \|P -  U\Sigma V^T \|_F =  (1+\delta) \min_{rank(B)
  \le d} \|P -  B\|_F .$$
In other words, the reconstructed matrix $U'\Sigma' V'^T $ is an $(1+\delta)$-approximation
to the best-$d$ approximation of $P$. \cite{clarkson2017low} shows that this
decomposition can be performed in $O\left(nnz(P) +
  nd^2/\delta^4\right)$ time. Therefore, the complexity of SRSVD on
our proximity matrix $P$ is bounded by $O\left({n \over \e}+
  nd^2/\delta^4\right)$. We set $\delta$ to be a constant so the
running time is bounded by $O\left({n \over \e}+
  nd^2\right)$.  Finally, we return $U'\sqrt{\Sigma'}$ and
$V'\sqrt{\Sigma'}$ as the content/context  embedding vectors.
Note that for undirected graph, the proximity
matrix $P$ is a symmetric matrix, in which case SVD on $P$ is equivalent
to eigendecomposition on $P$. % In this case, we have $U'=V'$, and hence
% $U'\sqrt{\Sigma'} = V'\sqrt{\Sigma'} $, 
% which means
% our algorithm only returns a single embedding vector for each node in
% a undirected graph. 

\header{\bf Running time and parallelism. } By
Lemma~\ref{lem:backward}, the total running time for the backward
push is $O\left(n\cdot \frac{m}{n\cdot
    \brmax}\right) = O\left(\frac{m}{
    \e}\right) $. Combining the running time $O\left({n \over \e}+
  nd^2\right)$  for randomized SVD, it follows that 
the running time of \strap is bounded by
$O\left({m \over \e} + nd^2\right)$.  We can provide tradeoffs
between scalability and accuracy by manipulating the
error parameter $\e$. In particular, as we decrease $\e$, we tradeoff
running time for more accurate
embeddings. In practice, the backward push part can be trivially
parallelized by running  backward push algorithms on multiple nodes
at the same time. To parallelize the SVD part, we use
frPCA~\cite{feng2018fast}, a parallel randomized SVD algorithm designed
for sparse matrices.  
% \begin{lemma}
%   \label{lem:strap_time}
%   The running time of \strap is bounded by $O\left(m/\e + nd^2\right)$.
% \end{lemma}

%%% Local Variables:
%%% mode: latex
%%% TeX-master: "paper"
%%% End:

%\input{analysis.tex}
%\input{related.tex}

 \begin{figure*}[!t] 
\begin{small}
 \centering
   \vspace{-5mm}
   % \begin{footnotesize}
   \includegraphics[height=28mm]{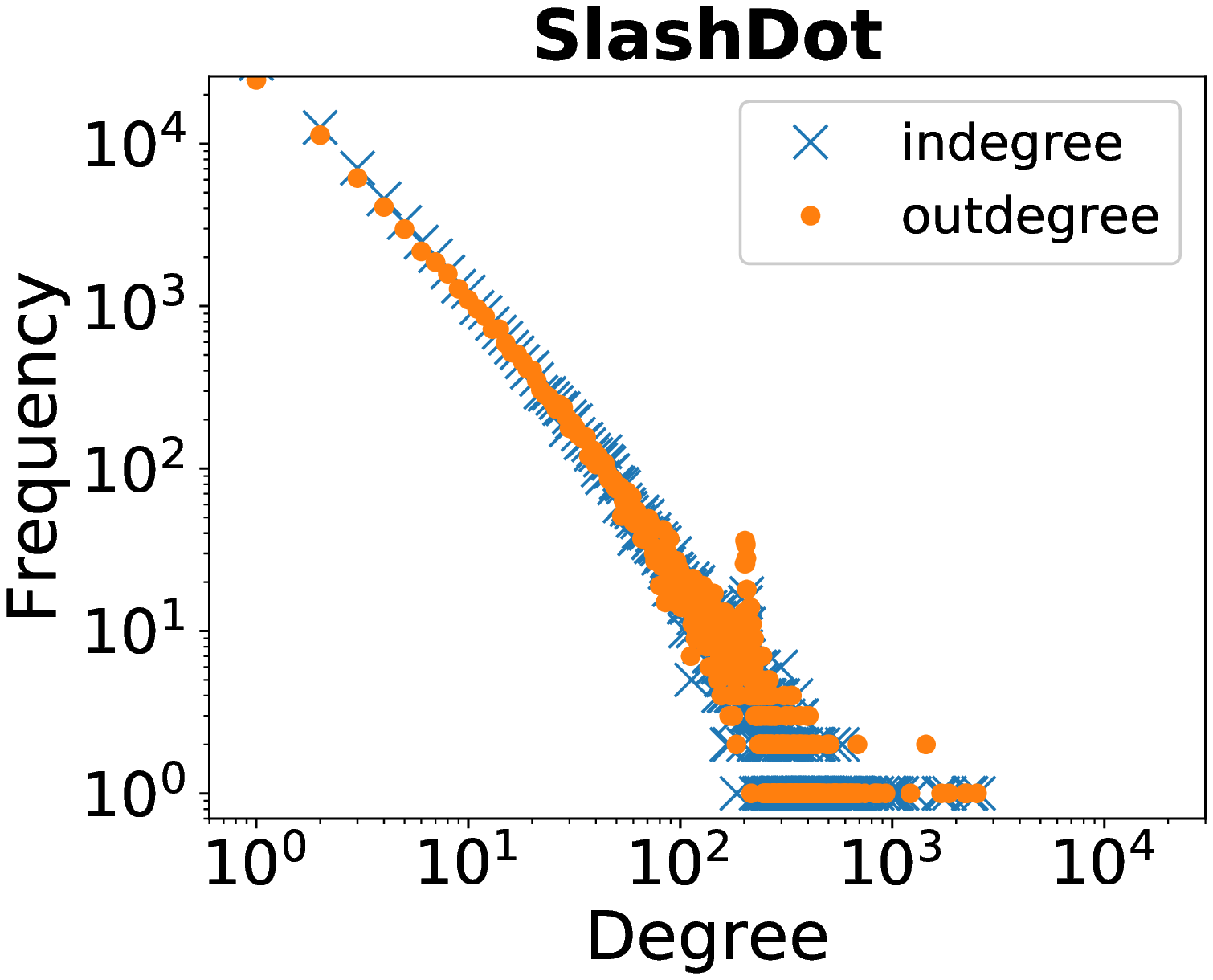}
 \includegraphics[height=28 mm]{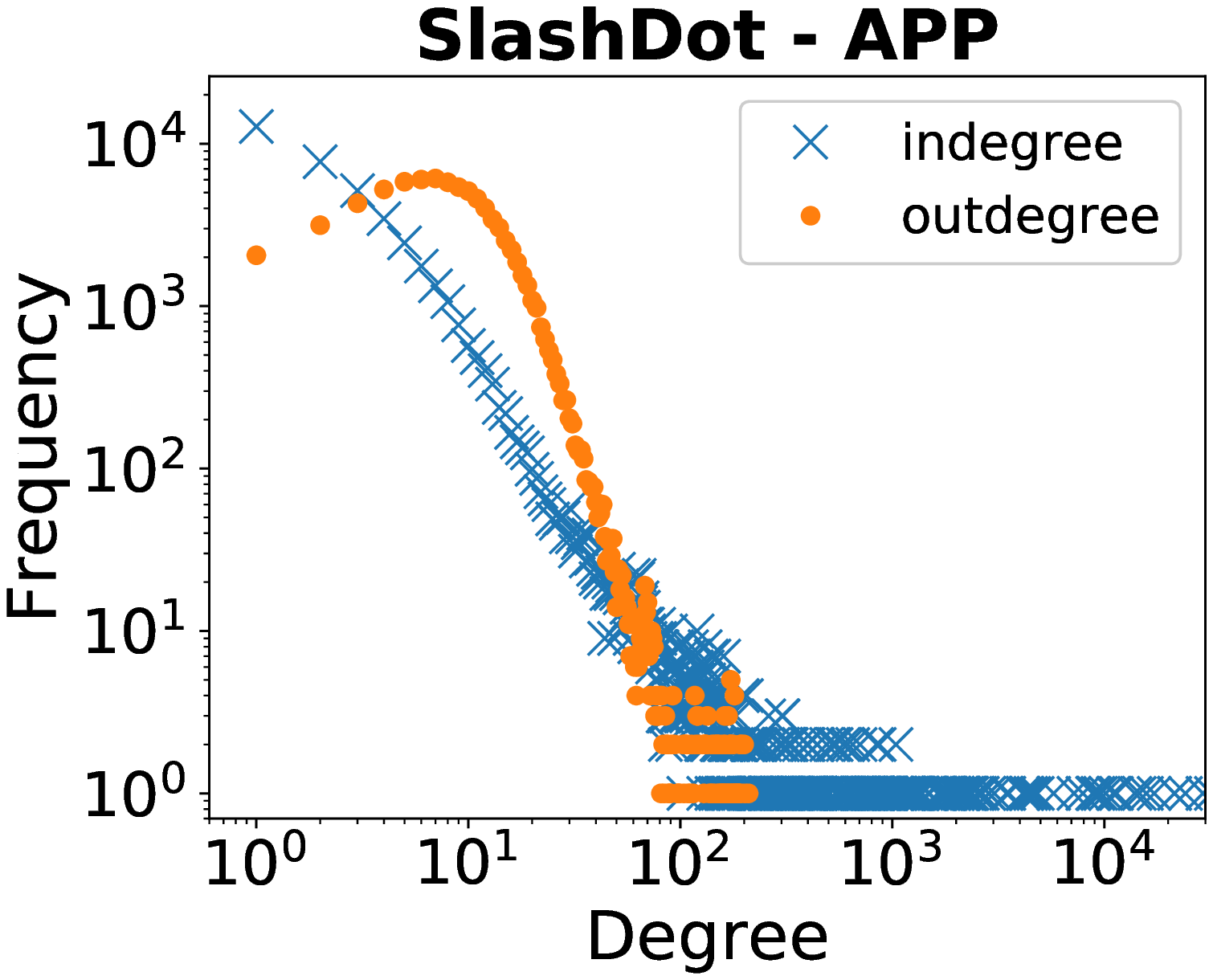}
 \includegraphics[height=28mm]{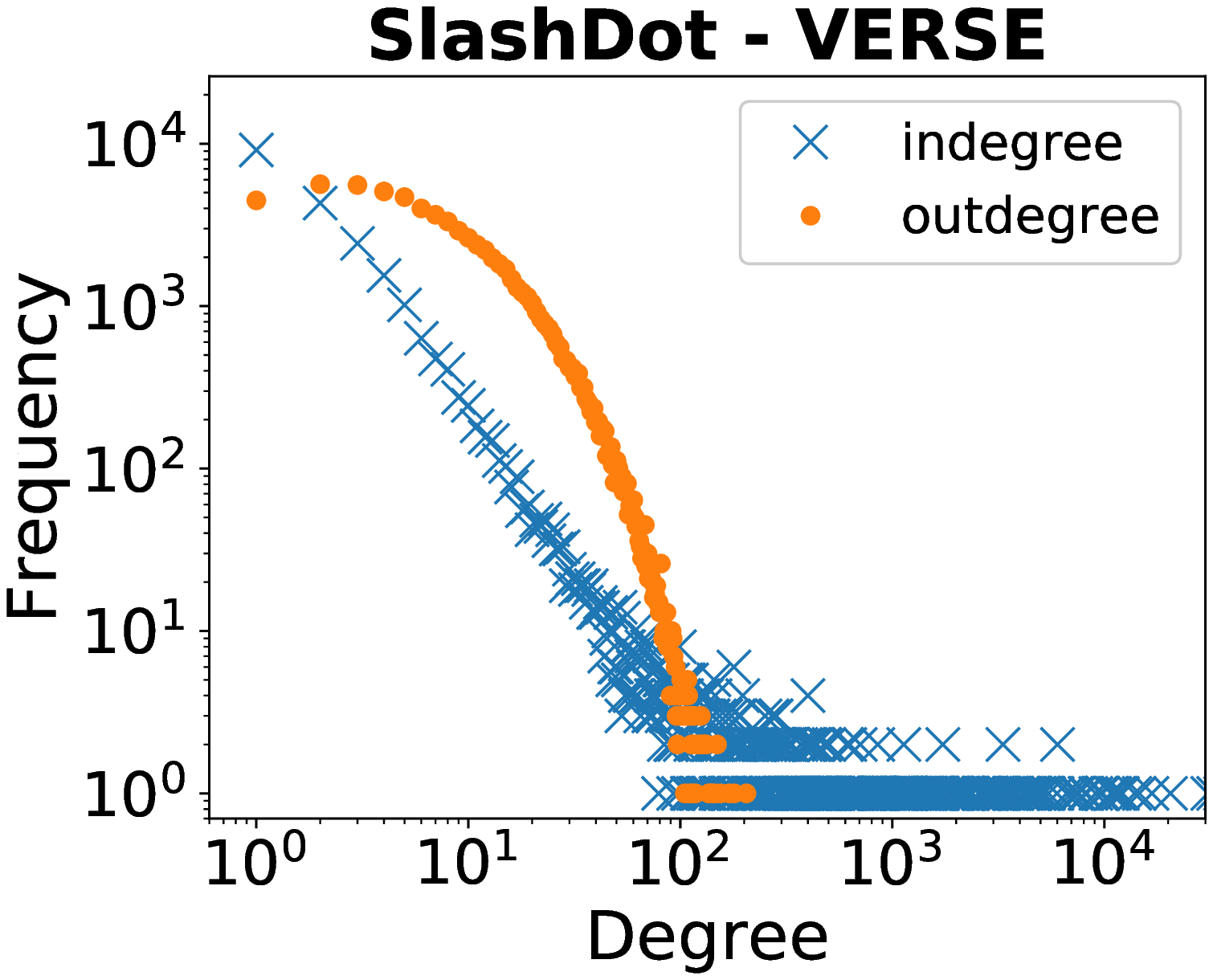}
 \includegraphics[height=28mm]{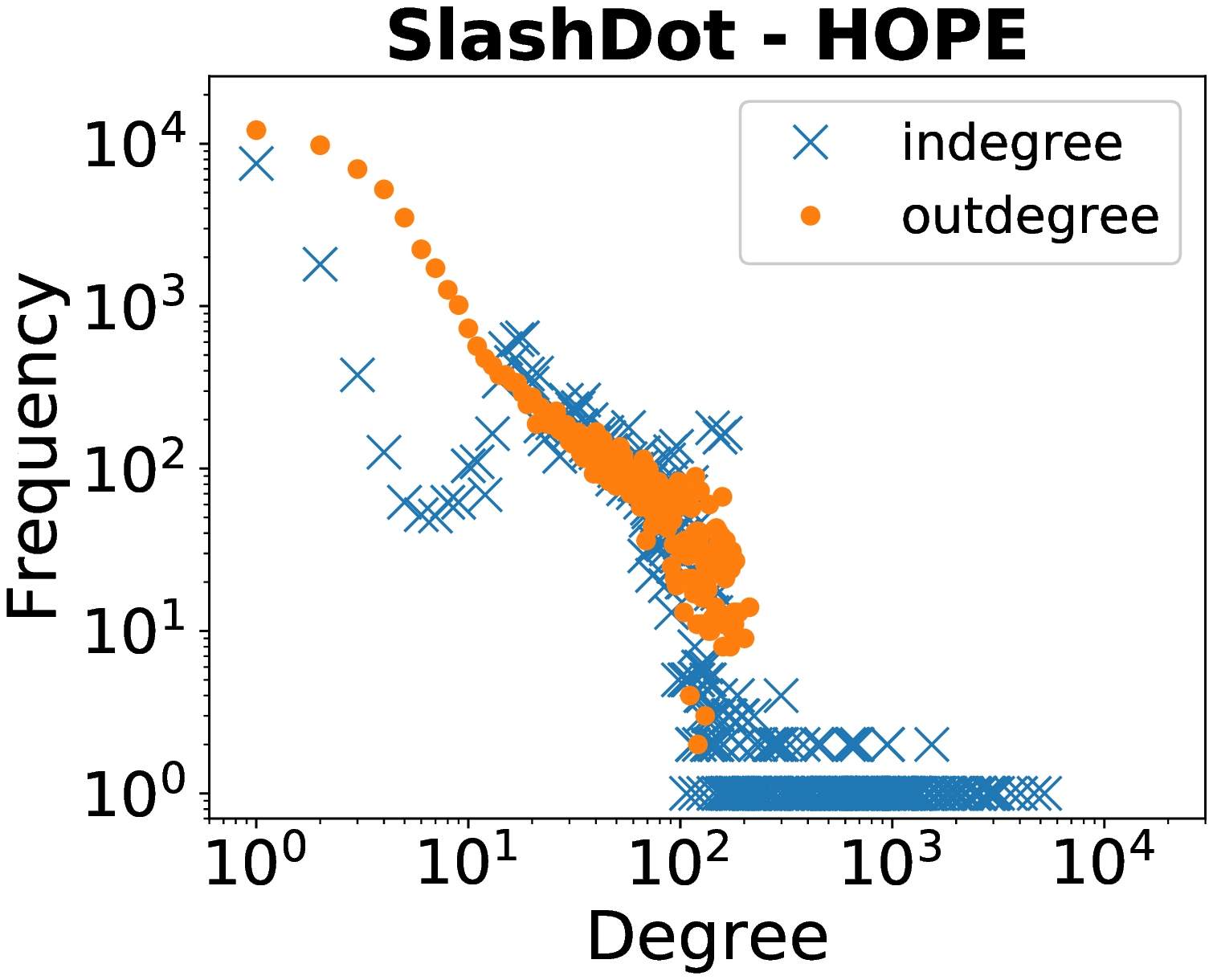}
  \includegraphics[height=28mm]{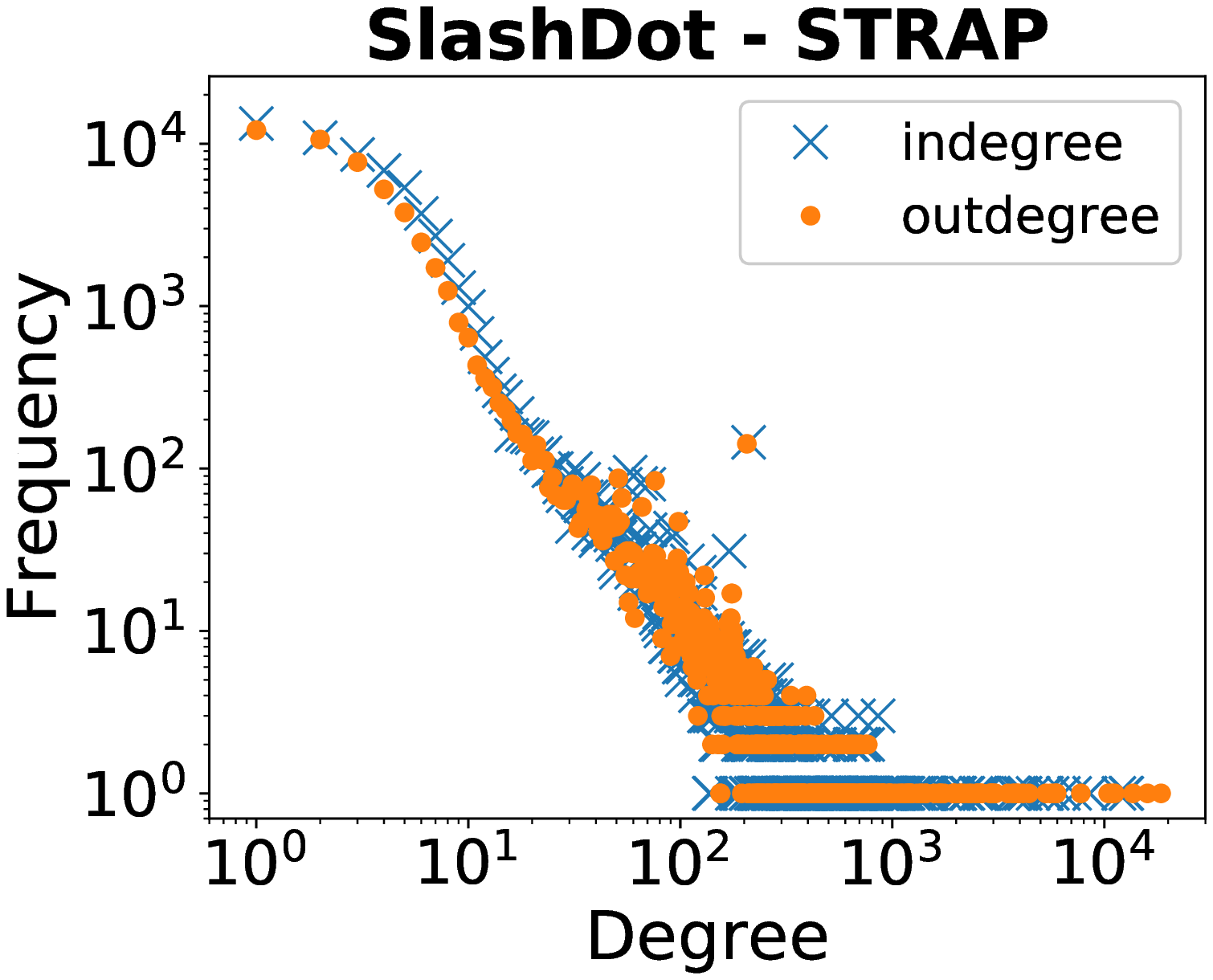}
  \vspace{-4mm}
 \caption{  Degree distributions of SlashDot.} \label{fig:degree_slashdot}
\vspace{-4mm}
\end{small}
\end{figure*}
\section{Experiments} \label{sec:exp}
This section experimentally evaluates \strap against
the states of the art. All experiments are conducted on a machine with
a Xeon(R) E7-4809@2.10GHz CPU and 320GB memory. 

%\vspace{0mm}
\subsection{Experimental Settings} \label{sec:exp-setting}
\header
{\bf Datasets.} We employ seven widely-used datasets, as shown in
Table~\ref{tbl:datasets}. BlogCatalog, Flickr and YouTube are three undirected
  social networks where nodes represent users and edges represent
  relationships between users. 
WikiVote is the directed Wikipedia who-votes-on-whom network.
Slashdot is the directed social network of Slashdot.com. Brazil and
Euro~\cite{ribeiro2017struc2vec} are two airport
networks with nodes as airports and edges as commercial airlines. All data sets are obtained from public
sources \cite{SNAP,LWA,NRC}.
\begin{table}[!t]
\centering
\vspace{-2mm}
\caption{Data Sets.}
\vspace{-3mm}
 \begin{tabular}{|l|l|r|r|} %p{1.3in}|}
 \hline
 {\bf Data Set} & {\bf Type} & {\bf $\boldsymbol{n}$} & {\bf
                                                        $\boldsymbol{m}$}	 \\ \hline
  BlogCatalog (BC) & undirected & 10,312	& 333,983 \\

   Flickr (FL) & undirected & 80,513 & 5,899,882 \\
 % {\color{blue}Com-LiveJournal(CL)} & undirected & 3,997,962	& 34,681,189 \\
 Youtube (YT) &	undirected	&	1,138,499	&
                                                            2,990,443
   \\
    WikiVote (WV)& directed & 7,115 & 103,689 \\

    Slashdot (SD) & directed & 82,168 & 870,161 \\
   Euro & undirected &  399 & 5,993 \\
   Brazil & undirected & 131 & 1,003 \\
 \hline
\end{tabular}
\label{tbl:datasets}
%\tbldown
\vspace{-5mm}
\end{table}
  % \item Cora: a citation network.  The datasets contain sparse
  % bag-of-words feature vectors for each document and a list of
  % citation links between documents
% \item is a network of social interactions among blog- gers in the BlogCatalog website. Node-labels represent topic categories provided by authors.

  \header
  {\bf Competitors and Parameter Setting.} Unless otherwise specified,
  we set the embedding dimensionality $d$ to $128$ in line
  with previous research~\cite{perozzi2014deepwalk,tsitsulin2018verse,
    zhang2018arbitrary}.  For \strap,
   we set the error parameter $\e = 0.00001$, so that the running
   time of our method is comparable to that of the fastest competitor. The decay factor
   $\alpha$ is set to be $0.5$ to balance the tradeoff between transductive and
   inductive effectiveness. We evaluate \strap against several state-of-the-art
  graph embedding algorithms.  We obtain the source code of these
  methods from GitHub and present their results with the authors' preferred parameters. 
  \begin{itemize}
\item DeepWalk\footnote{https://github.com/phanein/deepwalk}~\cite{perozzi2014deepwalk}   uses truncated random walks and the skip-gram model
  to learn embedding vectors. We use the parameters
  suggested in~\cite{perozzi2014deepwalk}: window size 10, walk length
  40, and the number of walks from each node to be 80. 
\item Node2Vec\footnote{https://github.com/aditya-grover/node2vec}~\cite{grover2016node2vec}  generalizes DeepWalk by adopting potentially biased
  random walks. We set the bias parameters $p=q=1$ and use the default settings for other parameters. 
\item
  HOPE\footnote{https://github.com/ZW-ZHANG/HOPE}~\cite{ou2016asymmetric}
  uses sparse SVD to decompose the proximity matrix of form $M_g^{-1}M_\ell$. 
  The default HOPE uses Katz similarity as its proximity measure. Since Katz
  does not converge on directed graphs with sink nodes, we evaluate
  HOPE with Personalized PageRank and set decay factor $\alpha = 0.5$
  as suggested in ~\cite{ou2016asymmetric}.
  
\item
  VERSE\footnote{https://github.com/xgfs/verse}~\cite{tsitsulin2018verse}
  is a random walk method that uses Personalized PageRank 
  and SimRank as the proximities. The paper presents two
  methods, VERSE, which simulates $\alpha$-discounted random walk to train
  the skip-gram model, and  fVERSE, which directly computes
  pair-wise PPR and SimRank.  We exclude fVERSE due to its
  $\Theta(n^2)$ complexity. Following the suggestion
  in~\cite{tsitsulin2018verse}, we set the number of iterations to be
  $10^5$. We set the decay factor $\alpha$ to the default value $0.85$~\cite{tsitsulin2018verse},
  as we have experienced performance drop for $\alpha = 0.5$.  VERSE supports directed graphs by  producing asymmetric content/context embedding vectors $\s$ and $\t$~\cite{tsitsulin2018verse}.

\item
  APP\footnote{https://github.com/AnryYang/APP}~\cite{zhou2017scalable}
  uses $\alpha$-discounted random walk to train asymmetric content/context embedding 
  vectors for each node using the skip-gram model. We set the number of
  samples per node to $200$ and the decay factor $\alpha$ to  $0.15$,
  as suggested in~\cite{zhou2017scalable}. % For similar reason
  % as VERSE, we set the decay factor $\alpha =0.15$,  the default value
  % recommended by the authors~\cite{zhou2017scalable}.
  
\item AROPE\footnote{https://github.com/ZW-ZHANG/AROPE}~\cite{zhang2018arbitrary} is a factorization method that preserves % higher order
  % proximities  by applying eigen-decomposition to the adjacency matrix
  % and reweighing largest eigen-values. AROPE is able to decompose
  the
  polynomial-shaped proximity matrix $P=\sum_{i=1}^q w_iA^i$. AROPE
  includes HOPE  and LINE as its special
  cases~\cite{zhang2018arbitrary} and achieves high scalability as it
  only performs eigen-decomposition to the (sparse) adjacency matrix. However, AROPE does not allow
  non-linear operations on the proximity matrix $P$, and it only works
  for undirected graphs due to the usage of eigen-decomposition. We
  use one of  the default
  set: $q =3$  and $w = \{1, 0.1, 0.01\}$. 
\end{itemize}

\header{\bf Remark.} Note that Node2Vec, VERSE, and AROPE are able to generate
multiple embedding vectors with varying parameters for each node. For example, AROPE
sets $q=1,2,3,-1$ and various weights $w$ to generate
multiple vectors for each node, and chooses the best-fit embedding vector for a specific
task using cross-validation.
Our method also allows multiple embeddings by a grid search on the decay factor $\alpha$
and error parameter $\e$. However, we argue that a fair comparison
should evaluate the embedding
vectors generated by a single set of parameters over various tasks. As
a counter-example, we note that one set of default parameters ($q=1$) in AROPE is to directly decompose
the adjacency matrix, which naturally gives the best graph
reconstruction result. However, the adjacency matrix performs poorly for inductive tasks such as link prediction or
node classification, and thus AROPE will select the embeddings
generated by a larger $q$ for these tasks. Our method can achieve similar results by
setting $\alpha$ close to $1$ for graph reconstruction and a smaller
$\alpha$ for link prediction or node classification (see Figure~\ref{fig:parameter_alpha}). However,
this would not be fair to other methods with fixed
parameters. Therefore, we believe that the only fair way to compare these
methods is to evaluate them in both
transductive and inductive tasks with consistent parameters. We also
include ADJ-SVD, the method that directly applies SVD to the adjacency
matrix, to demonstrate the above argument. For each task, we run each method ten
times and report the average of its score.

\vspace{-2mm}
\subsection{Graph Reconstruction}
We perform graph
reconstruction task to see if the low-dimensional representation can
accurately reconstruct the adjacency matrix. For
each method, we train embedding vectors and rank pairs
of nodes $(u,v)$ according to the inner product $\s_u \cdot \t_v$,
where $\s_u$ and $\t_v$ are the content and context embedding vectors of
node $u$ and $v$, respectively. We then take the top-$m$  pairs of nodes (removing
self-loop) to reconstruct the graph, where $m$ is the number of edges in the original graph.

\header{\bf Degree distributions on directed graphs.}
Figure~\ref{fig:degree_slashdot} shows the degree distribution of the
directed graph SlashDot and the reconstructed graphs by HOPE, APP,
VERSE and \strap.  We exclude Node2Vec and DeepWalk as they generate
identical in- and out-degree distributions. Similar to the results on
WikiVote, \strap is
the only method that can generate out-degree distribution 
similar to that of the original graph, which concurs with our
theoretical analysis for transpose proximity.

\header{\bf Reconstruction precision.} We calculate the ratio of
real links in top-$m$ predictions as the reconstruction
precision. Table~\ref{tbl:nr} shows the results the reconstruction
precision for 
each dataset. As expected, ADJ-SVD achieves the highest precision on all
graphs. For other methods, we observe that  \strap significantly outperforms all
existing methods. The advantage of \strap becomes more obvious on
directed graphs WikiVote and SlashDot, which demonstrates the effectiveness of transpose
proximity. 

% However, as we argued before, network reconstruction
% results only reflects the transductive strength of an algorithm, as
% Therefore, it
% remains to be seen that if \strap is able to achieve similar quality
% on inductive tasks.
  
\begin{table}[h]
  \begin{small}
  \centering
     \vspace{-2mm}
\caption{Graph Reconstruction Precision (\%).}
   \vspace{-3mm}
 \begin{tabular}{l >{\centering\arraybackslash}p{0.35in}
   >{\centering\arraybackslash}p{0.35in}
   >{\centering\arraybackslash}p{0.35in}
   >{\centering\arraybackslash}p{0.35in}  >{\centering\arraybackslash}p{0.35in}} %p{1.3in}|}
 \hline
 {\bf Method} & BC & FL  & YT & WV	& SD \\ \hline
   DeepWalk & 5.08  & 4.86 & 0.68 & 1.64 & 3.45\\
   Node2Vec & 6.53  & 2.85 & 0.13 & 4.19 & 0.15\\
      HOPE & 21.85  & 14.90 & 8.78 & 10.98 & 8.61\\
   APP & 18.50  & 19.95 & 12.34 & 10.85 & 11.91\\
   VERSE & 40.03  & 20.22 & 6.09 & 20.89 & 10.73\\   
   AROPE & 37.06  & 26.21 & 24.50 & NA & NA\\
   \strap & \cellcolor{blue!25} 52.32  &  \cellcolor{blue!25} 34.92 &
                                                                      \cellcolor{blue!25}
                                                                      27.18
                              &  \cellcolor{blue!25} 55.29 &
                                                             \cellcolor{blue!25}  24.42\\
    \hline
   ADJ-SVD & 59.53  & 41.34 & 31.81 & 74.15 & 30.87\\
 \hline
\end{tabular}
\label{tbl:nr}
\vspace{-4mm}
\end{small}
\end{table}

% Since our methods ``knows'' the degree distribution, we also perform
% graph reconstruction under the assumption that each method  knows the
% degree distribution. More precisely,   for each node $u$, we
% reconstruct the neighbors of $u$ by connecting $u$ to
% the top-$\dout(u)$ nodes $v$ with highest inner products $\s_u \cdot
% \t_v$  for  $v\neq u$. Table~\ref{tbl:nr_degree} report the reconstruction
% precision. We observe that even with the degree information, our
% method still outperforms the competitors. 

% \begin{table}[h]
%   \centering
%       %\vspace{0mm}
% \caption{Graph Reconstruction with known degrees.}
%     %\vspace{-3mm}
%  \begin{tabular}{l >{\centering\arraybackslash}p{0.35in}
%    >{\centering\arraybackslash}p{0.35in}
%    >{\centering\arraybackslash}p{0.35in}
%    >{\centering\arraybackslash}p{0.35in}  >{\centering\arraybackslash}p{0.35in}} %p{1.3in}|}
%  \hline
%  {\bf Method} & BC & FL  & YT & WV	& SD \\ \hline
%    \strap & 91.34  & 91.34 & 91.34	& 91.34 & 91.34\\
%    DeepWalk & 91.34  & 91.34 & 91.34	& 91.34 & 91.34\\
%    Node2Vec & 91.34  & 91.34 & 91.34	& 91.34 & 91.34\\
%    HOPE & 91.34  & 91.34 & 91.34	& 91.34 & 91.34\\
%    APP & 91.34  & 91.34 & 91.34	& 91.34 & 91.34\\
%    VERSE & 91.34  & 91.34 & 91.34	& 91.34 & 91.34\\   
%     AROPE & 91.34  & 91.34 & 91.34	& 91.34 & 91.34\\
%  \hline
% \end{tabular}
% \label{tbl:nr_degree}
% %\vspace{-2mm}
% \end{table}

\subsection{Link Prediction}
An important inductive application of graph embedding is predicting unobserved links in the graph.
To test the performance of different embedding methods on this task,
we randomly hide $50\%$ of the edges as positive samples for testing 
and sample the same number of non-existing edges as negative examples.
We then train embedding
vectors on the rest of the $50\%$ edges and predict the most likely
edges which are not observed in the training data from the learned
embedding. Table~\ref{tbl:lp} reports the precision of each
method. We observe that  \strap is consistently the best
predictor on all datasets except YouTube, on which VERSE takes the
lead by 1\%. We also note that \strap significantly outperforms the
state-of-the-art factorization methods AROPE and HOPE, and we attribute
this quality to the non-linearity of our methods.

\begin{table}[h]
    \begin{small}
  \centering
      \vspace{-2mm}
\caption{Link Prediction Precision (\%).}
    \vspace{-3mm}
 \begin{tabular}{l >{\centering\arraybackslash}p{0.35in}
      >{\centering\arraybackslash}p{0.35in}
   >{\centering\arraybackslash}p{0.35in}
   >{\centering\arraybackslash}p{0.35in}  >{\centering\arraybackslash}p{0.35in}} %p{1.3in}|}
 \hline
 {\bf Method} & BC & FL  & YT & WV	& SD \\ \hline
   DeepWalk & 53.59  & 70.26 & 63.46 & 66.72 & 65.42\\
   Node2Vec & 63.58  & 57.26 & 54.55 & 56.81 & 53.37\\
   HOPE & 79.97  & 86.75 & 67.23 & 85.67 & 84.11\\
   APP & 78.19   & 81.69 & 63.11 & 61.44 & 72.77\\
   VERSE & 87.99  & 90.13 & \cellcolor{blue!25} 67.51 & 86.39 & 83.99\\   
   AROPE & 88.09  & 88.78 & 65.43 & NA & NA\\
    \strap & \cellcolor{blue!25} 88.92  & \cellcolor{blue!25} 91.49 & 66.86 & \cellcolor{blue!25} 91.79 & \cellcolor{blue!25} 84.47\\
   \hline
   ADJ-SVD & 76.36  & 89.27 & 59.31 & 74.02 & 62.77\\
 \hline
\end{tabular}
\label{tbl:lp}
\vspace{-2mm}
\end{small}
\end{table}
To demonstrate  the effect of the training ratio, we
also report the precisions of each method with varying
training/testing ratio on BlogCatalog. We observe that our method consistently outperforms
existing methods for all training ratios, with VERSE being the closest competitor.%  A
% surprising observation is that DeepWalk and Node2Vec achieve highest precision
% with 10\% training data, which is strong evidence that DeepWalk and Node2Vec
% trade transductive strength for inductive strength.

% \begin{table}[h]
% \centering
% \caption{Link Prediction Results for WikiVote.}
%  \begin{tabular}{l >{\centering\arraybackslash}p{0.35in}
%    >{\centering\arraybackslash}p{0.35in}
%    >{\centering\arraybackslash}p{0.35in}
%    >{\centering\arraybackslash}p{0.35in}  >{\centering\arraybackslash}p{0.35in}} %p{1.3in}|}
%  \hline
%  {\bf Method} & $10\%$ & $30\%$  & $50\%$ & $70\%$	& $90\%$ \\ \hline
%    \strap & 91.34  & 91.34 & 91.34	& 91.34 & 91.34\\
%    DeepWalk & 91.34  & 91.34 & 91.34	& 91.34 & 91.34\\
%    Node2Vec & 91.34  & 91.34 & 91.34	& 91.34 & 91.34\\
%    HOPE & 91.34  & 91.34 & 91.34	& 91.34 & 91.34\\
%    APP & 91.34  & 91.34 & 91.34	& 91.34 & 91.34\\
%    VERSE & 91.34  & 91.34 & 91.34	& 91.34 & 91.34\\   
%     AROPE & 91.34  & 91.34 & 91.34	& 91.34 & 91.34\\
%  \hline
% \end{tabular}
% \label{tbl:lp_wikivote}
% %\vspace{-0mm}
% \end{table}

\begin{table}[h]
    \begin{small}
  \centering
      \vspace{0mm}
\caption{Link Prediction Precision (\%) for BlogCatalog.}
    \vspace{-3mm}
 \begin{tabular}{l >{\centering\arraybackslash}p{0.35in}
   >{\centering\arraybackslash}p{0.35in}
   >{\centering\arraybackslash}p{0.35in}
   >{\centering\arraybackslash}p{0.35in}  >{\centering\arraybackslash}p{0.35in}} %p{1.3in}|}
 \hline
 {\bf Method} & $10\%$ & $30\%$  & $50\%$ & $70\%$	& $90\%$ \\
   \hline
   DeepWalk & 61.77  & 54.62 & 53.59 & 53.53 & 53.41\\
   Node2Vec & 57.32  & 63.77 & 63.58 & 62.57 & 63.66\\
   HOPE & 68.24  & 72.67 & 79.97 & 81.63 & 83.45\\
   APP & 53.49  & 70.91 & 78.19 & 77.31 & 78.67\\
   VERSE & 83.73  & 86.38 & 87.99 & 88.74 & 89.52\\   
   AROPE & 80.77  & 87.37 & 88.09 & 88.35 & 88.49\\
      \strap &  \cellcolor{blue!25} 84.78  &  \cellcolor{blue!25} 87.40 &  \cellcolor{blue!25} 88.92 &  \cellcolor{blue!25} 89.92 &  \cellcolor{blue!25} 90.42\\
   \hline
   ADJ-SVD & 57.12  & 72.86 & 76.36 & 80.36 & 83.30\\
   \hline
\end{tabular}
\label{tbl:lp_blogcatalog}
\vspace{-2mm}
  \end{small}
\end{table}

\subsection{Node Classification}
Node classification aims to predict the correct node labels in a graph. 
Following the same experimental procedure in~\cite{perozzi2014deepwalk}, we
randomly sample a portion of labeled vertices for training and use the
rest for testing.  The
training ratio is varied from 10\% to 90\%.
We use LIBLINEAR~\cite{fan2008liblinear} to perform logistic regression with default parameter settings.
To avoid the
thresholding effect~\cite{tang2009large}, we assume
that the number of labels for test
data is given~\cite{perozzi2014deepwalk}. The performance of each
method is evaluated in terms of average Micro-F1 and
average Macro-F1~\cite{tsoumakas2009mining}, and we only report Micro-F1 as we experience similar behaviors with
Macro-F1. Table~\ref{tbl:nc_blog} and Table~\ref{tbl:nc_flickr} show the
node classification results on BlogCatalog and Flickr. Surprisingly, DeepWalk
outperforms all successors other than \strap on both datasets. On the
other hand, \strap is able to achieve comparable precision to that of
DeepWalk. In particular, \strap significantly outperforms HOPE and
AROPE, which again demonstrates the effectiveness of the non-linearity. 

\begin{table}[h]
    \begin{small}
  \centering
      \vspace{-2mm}
\caption{Node Classification on BlogCatalog.}
    \vspace{-3mm}
 \begin{tabular}{l >{\centering\arraybackslash}p{0.35in}
   >{\centering\arraybackslash}p{0.35in}
   >{\centering\arraybackslash}p{0.35in}
   >{\centering\arraybackslash}p{0.35in}  >{\centering\arraybackslash}p{0.35in}} %p{1.3in}|}
 \hline
 {\bf Method} & $10\%$ & $30\%$  & $50\%$ & $70\%$	& $90\%$  \\ \hline
   DeepWalk & 35.93  & 39.65 & 40.86 & 41.93 &  \cellcolor{blue!25} 43.31\\
   Node2Vec & 34.60  & 38.27 & 39.31 & 40.14 & 40.36\\
   HOPE & 16.68  & 17.85 & 17.92 & 19.23 & 20.18\\
   APP & 28.09  & 31.63 & 33.31 & 33.71 & 33.49\\
   VERSE & 31.48  & 35.96 & 38.32 & 39.64 & 40.49\\   
   AROPE & 27.01  & 30.98 & 31.89 & 32.76 & 32.94\\
     \strap &  \cellcolor{blue!25} 36.42  &  \cellcolor{blue!25} 40.29 &  \cellcolor{blue!25} 41.59 &  \cellcolor{blue!25} 42.68 & 42.55\\
   \hline
      ADJ-SVD & 23.15  & 28.42 & 29.75 & 31.83 & 31.85\\
\hline
\end{tabular}
\label{tbl:nc_blog}
\vspace{-2mm}
\end{small}
\end{table}

\begin{table}[h]
    \begin{small}
  \centering
      \vspace{-2mm}
\caption{Node Classification on Flickr.}
    \vspace{-3mm}
 \begin{tabular}{l >{\centering\arraybackslash}p{0.35in}
   >{\centering\arraybackslash}p{0.35in}
   >{\centering\arraybackslash}p{0.35in}
   >{\centering\arraybackslash}p{0.35in}  >{\centering\arraybackslash}p{0.35in}} %p{1.3in}|}
 \hline
 {\bf Method} & $10\%$ & $30\%$  & $50\%$ & $70\%$	& $90\%$  \\ \hline
   DeepWalk & 38.96  & 40.83 &\cellcolor{blue!25}  41.54 & \cellcolor{blue!25} 41.85 & \cellcolor{blue!25} 42.08\\
   Node2Vec & 38.15  & 39.85 & 40.60 & 41.06 & 41.34\\
   HOPE & 16.39  & 16.59 & 16.59 & 16.67 & 16.56\\
   APP & 33.15  & 35.29 & 35.99 & 36.23 & 36.54\\
   VERSE & 34.54  & 37.10 & 38.07 & 38.57 & 38.83\\   
    AROPE & 29.56  & 30.62 & 30.89 & 31.27 & 31.73\\
 \strap & \cellcolor{blue!25} 39.32  & \cellcolor{blue!25} 41.00 & 41.47 & 41.77 & 42.06\\

   \hline
         ADJ-SVD & 24.52  & 26.59 & 27.22 & 27.54 & 27.97\\
 \hline

\end{tabular}
\label{tbl:nc_flickr}
\vspace{-2mm}
\end{small}
\end{table}

% \begin{table}[h]
% \centering
% \caption{Node Classification on Cora.}
%  \begin{tabular}{l >{\centering\arraybackslash}p{0.35in}
%    >{\centering\arraybackslash}p{0.35in}
%    >{\centering\arraybackslash}p{0.35in}
%    >{\centering\arraybackslash}p{0.35in}  >{\centering\arraybackslash}p{0.35in}} %p{1.3in}|}
%  \hline
%  {\bf Method} & $10\%$ & $30\%$  & $50\%$ & $70\%$	& $90\%$  \\ \hline
%    \strap & 91.34  & 91.34 & 91.34	& 91.34 & 91.34\\
%    DeepWalk & 91.34  & 91.34 & 91.34	& 91.34 & 91.34\\
%    Node2Vec & 91.34  & 91.34 & 91.34	& 91.34 & 91.34\\
%    HOPE & 91.34  & 91.34 & 91.34	& 91.34 & 91.34\\
%    APP & 91.34  & 91.34 & 91.34	& 91.34 & 91.34\\
%    VERSE & 91.34  & 91.34 & 91.34	& 91.34 & 91.34\\   
%     AROPE & 91.34  & 91.34 & 91.34	& 91.34 & 91.34\\
%  \hline
% \end{tabular}
% \label{tbl:nc_cora}
% %\vspace{-0mm}
% \end{table}

\header{\bf Node Structural Role Classification.}
We also perform node structural role classification
task~\cite{zhang2018arbitrary,ribeiro2017struc2vec} on Brazil and Euro, two airport
networks with nodes as airports and edges as commercial airlines. The goal is to assign each node
a label from 1 to 4 to indicate  the level of activities of the
corresponding airports. Due to the size of the graphs, we set the
dimension $d = 16$ for this particular task. Table~\ref{tbl:nc_brazil}
and Table~\ref{tbl:nc_eu}
shows the node structural role classification results on the two
graphs, respectively. Again, our method performs comparably well. This suggests that \strap preserves the structural
role of the graphs.  We also observe that DeepWalk, the main competitor
in the previous task,  achieves unsatisfying results, while
our method performs consistently on two very different tasks.

\begin{table}[h]
    \begin{small}
  \centering
      \vspace{-1mm}
\caption{Node Structural Role Classification on  Brazil.}
    \vspace{-3mm}
 \begin{tabular}{l >{\centering\arraybackslash}p{0.35in}
   >{\centering\arraybackslash}p{0.35in}
   >{\centering\arraybackslash}p{0.35in}
   >{\centering\arraybackslash}p{0.35in}  >{\centering\arraybackslash}p{0.35in}} %p{1.3in}|}
 \hline
 {\bf Method} & $10\%$ & $30\%$  & $50\%$ & $70\%$	& $90\%$\\ \hline 
   DeepWalk & 25.42  & 32.61 & 27.27 & 25.00 & 21.43\\
   Node2Vec & 36.44  & 41.30 & 42.42 & 37.50 & 50.00\\
   HOPE & 22.88  & 20.65 & 21.21 & 32.50 & 28.57\\
   APP & 24.58  & 35.87 & 36.36 & 40.00 & 28.57\\
   VERSE & 30.51  & 32.61 & 31.82 & 42.50 & 35.71\\   
   AROPE & \cellcolor{blue!25} 39.83  & 47.83 & 50.00 & \cellcolor{blue!25} 60.00 & 64.29\\
       \strap & 37.29  &\cellcolor{blue!25} 51.74 & \cellcolor{blue!25} 52.42 & 59.50 & \cellcolor{blue!25} 70.71\\
   \hline
     ADJ-SVD & 38.98  & 43.48 & 46.97 & 62.50 & 64.29\\
 \hline

\end{tabular}
\label{tbl:nc_brazil}
\vspace{-1mm}
\end{small}
\end{table}

\begin{table}[h]
    \begin{small}
  \centering
      \vspace{0mm}
\caption{Node Structural Role Classification on  Euro.}
    \vspace{-3mm}
 \begin{tabular}{l >{\centering\arraybackslash}p{0.35in}
   >{\centering\arraybackslash}p{0.35in}
   >{\centering\arraybackslash}p{0.35in}
   >{\centering\arraybackslash}p{0.35in}  >{\centering\arraybackslash}p{0.35in}} %p{1.3in}|}
 \hline
 {\bf Method} & $10\%$ & $30\%$  & $50\%$ & $70\%$	& $90\%$\\ \hline

   DeepWalk & 26.94  & 26.79 & 24.00 & 30.00 & 35.00\\
   Node2Vec & 37.78  & 40.00 & 39.00 & 40.83 & 50.00\\
   HOPE & 25.00  & 27.50 & 20.50 & 23.33 & 30.00\\
   APP & 26.11  & 32.14 & 28.50 & 38.33 & 42.50\\
   VERSE & 33.89  & 39.29 & 43.50 & 45.83 & 42.50\\   
   AROPE & 42.50  & 41.43 & 41.50 & 60.83 & 65.00\\
      \strap & \cellcolor{blue!25} 47.56  & \cellcolor{blue!25} 44.79 & \cellcolor{blue!25} 48.75 & \cellcolor{blue!25} 61.42 & \cellcolor{blue!25} 65.50\\
   \hline
     ADJ-SVD & 42.78  & 43.57 & 43.50 & 53.33 & 65.00\\
 \hline

\end{tabular}
\label{tbl:nc_eu}
\vspace{-1mm}
  \end{small}
\end{table}
% In conclusion, our methods achieves comparable performance in the
% traditional node classification tasks as the best competitor DeepWalk, and
% outperforms AROPE by a large margin. On the other hand, \strap comparable performance in the
% node structural role classification tasks as the best competitor
% AROPE, and significantly outperforms DeepWalk.

\subsection{Running Time and Scalability}
Table~\ref{tbl:time} reports the wall-clock time of each method, with
thread number bounded by 24. 
In general, our method achieves the same level of scalability as AROPE
does, and significantly outperforms all random walk methods. 
\begin{table}[h]
    \begin{small}
  \centering
    \vspace{-3mm}
    \caption{Running time (s).}
    \vspace{-3mm}
\begin{tabular}{l >{\centering\arraybackslash}p{0.35in}
   >{\centering\arraybackslash}p{0.35in}
   >{\centering\arraybackslash}p{0.35in}
   >{\centering\arraybackslash}p{0.35in}  >{\centering\arraybackslash}p{0.35in}} %p{1.3in}|}
 \hline
 {\bf Method} & BC & FL  & YT & WV	& SD \\ \hline
   DeepWalk & 1.2e3  & 1.3e4 & 1.7e5 & 7.3e2 & 1.2e4\\
  Node2Vec & 2.8e2  & 6.4e4 & 3.4e4 & 1.1e2 & 6.2e3\\
  HOPE & 3.5e2  & 2.5e3 & 1.9e5 & 2.3e2 & 2.5e3\\
   APP & 8.9e2  & 7.2e3 & 1.7e5 & 6.2e2 & 9.3e3\\
   VERSE & 2.7e2  & 2.4e3 & 3.6e4 & 1.1e2 & 1.7e3\\   
  AROPE & \cellcolor{blue!25} 2.4e1  & \cellcolor{blue!25} 1.3e2 & \cellcolor{blue!25} 1.0e3 & NA & NA\\
  \strap & 3.9e1 & 7.5e2 & 2.1e3 &\cellcolor{blue!25} 6.0e0 & \cellcolor{blue!25} 2.4e2\\

 \hline
\end{tabular}
\label{tbl:time}
\vspace{-2mm}
\end{small}
\end{table}

% We also conduct experiments to verify the scalability of our method. We
% apply our method to synthetic networks of different sizes generated by
% hyperbolic power-law graph
% generator~\cite{krioukov2010hyperbolic,aldecoa2015hyperbolic}.
% The result shows that our method computes the embedding vectors on
% million-node graphs using less than 30 minutes, and thus is scalable on
% large graphs.
%  \begin{figure}[h]
% \begin{small}
%  \centering
%    \vspace{-3mm}
%    % \begin{footnotesize}
%  %\includegraphics[height=31mm]{./Figs/alpha-nr.eps}
%   \includegraphics[height=28mm]{./Figs/nodes-time.eps}
%   \vspace{-5mm}
%  \caption{Scalability} \label{fig:scalability}
% \vspace{-5mm}
% \end{small}
% \end{figure}

\subsection{Parameter Analysis}

We study the effect of decay factor $\alpha$ and error parameter
$\e$. Figure~\ref{fig:parameter_alpha} shows how graph reconstruction
and link prediction precisions behave as we vary $\alpha$ from 1 to
0. The results show that  $\alpha$ provides tradeoffs
between inductive and transductive effectiveness: for $\alpha$ close to
$1$, the proximity matrix focuses on one-hop neighbors and thus
preserves the adjacency information.
As 
$\alpha$ approaches $0$, the information of multi-hop neighbors will be added
to the proximity matrix, trading transductive strength for  inductive
strength. Figure~\ref{fig:parameter_eps} shows how running time and  graph reconstruction
precisions behave as we vary error parameter $\e$.
It shows that $\e$ controls the tradeoff between precision
and running time. As we decrease $\e$, we tradeoff
running time for more accurate
embedding vectors.
 \begin{figure}[h]
\begin{small}
 \centering
   \vspace{0mm}
   % \begin{footnotesize}
 \includegraphics[height=31mm]{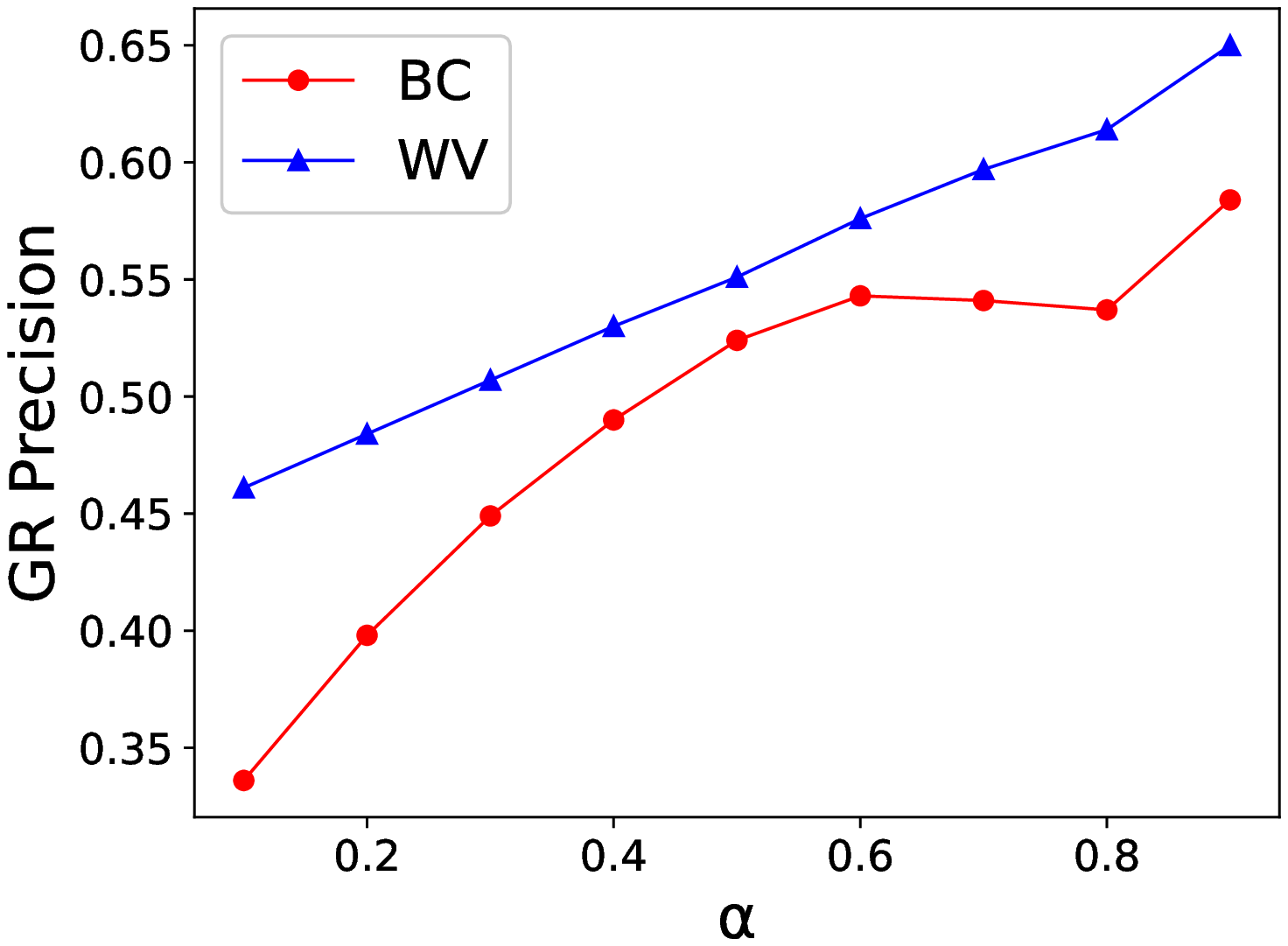}
  \includegraphics[height=31mm]{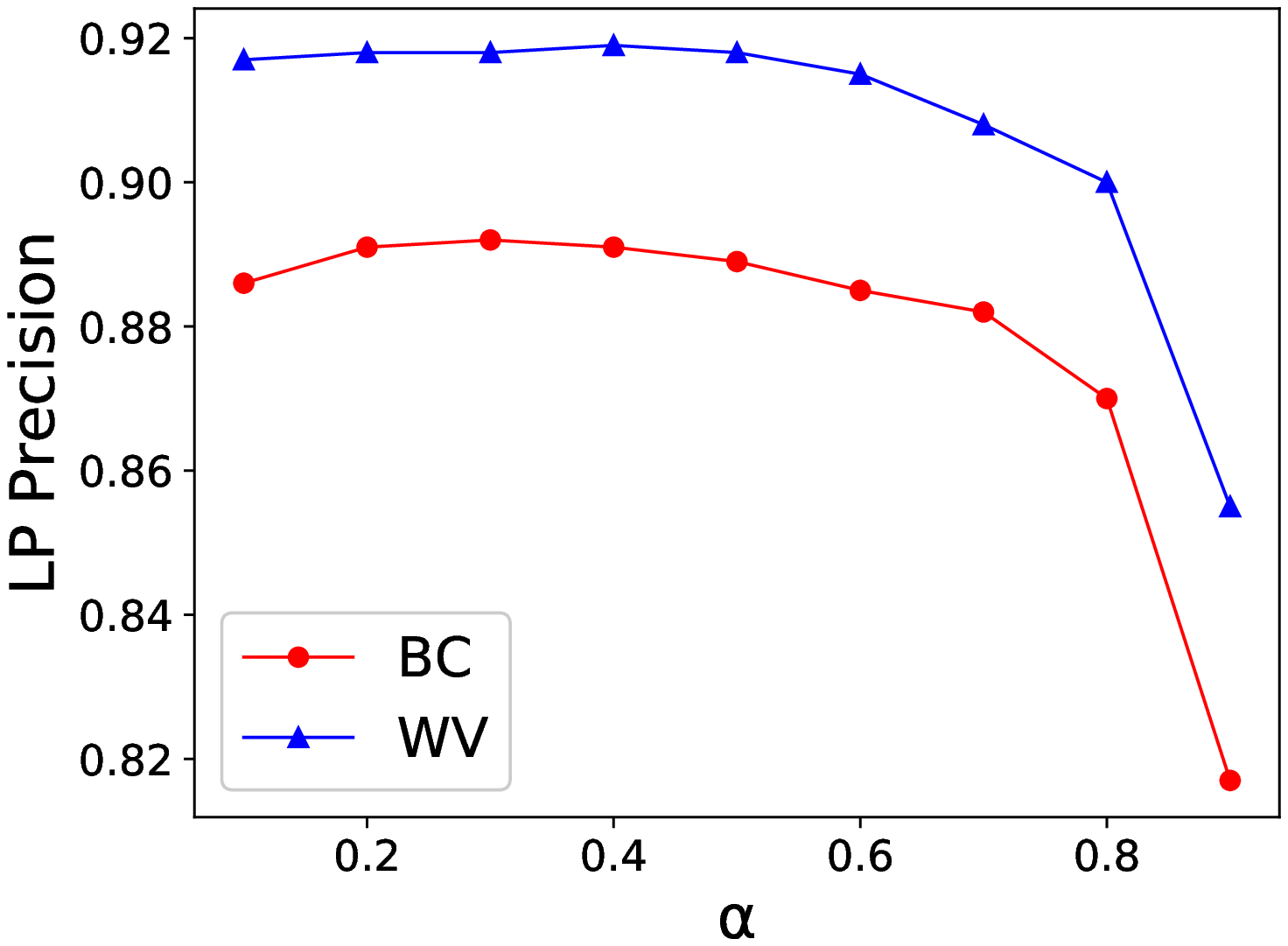}
  \vspace{-4mm}
 \caption{Graph reconstruction and link prediction precisions with varying $\alpha$.} \label{fig:parameter_alpha}
\vspace{-2mm}
\end{small}
\end{figure}

 \begin{figure}[h]
\begin{small}
 \centering
   \vspace{-2mm}
   % \begin{footnotesize}
  \includegraphics[height=31mm]{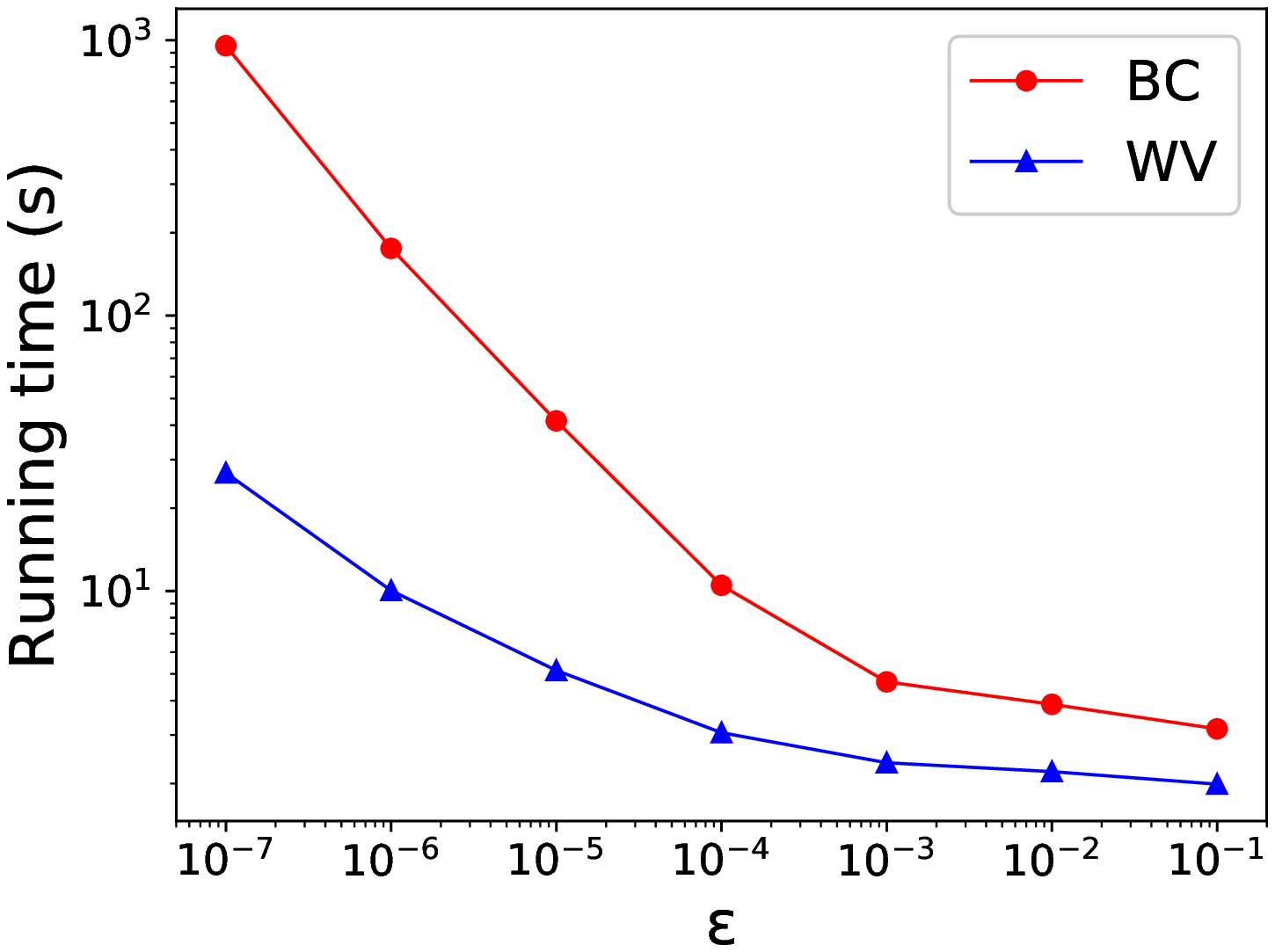}
 \includegraphics[height=31mm]{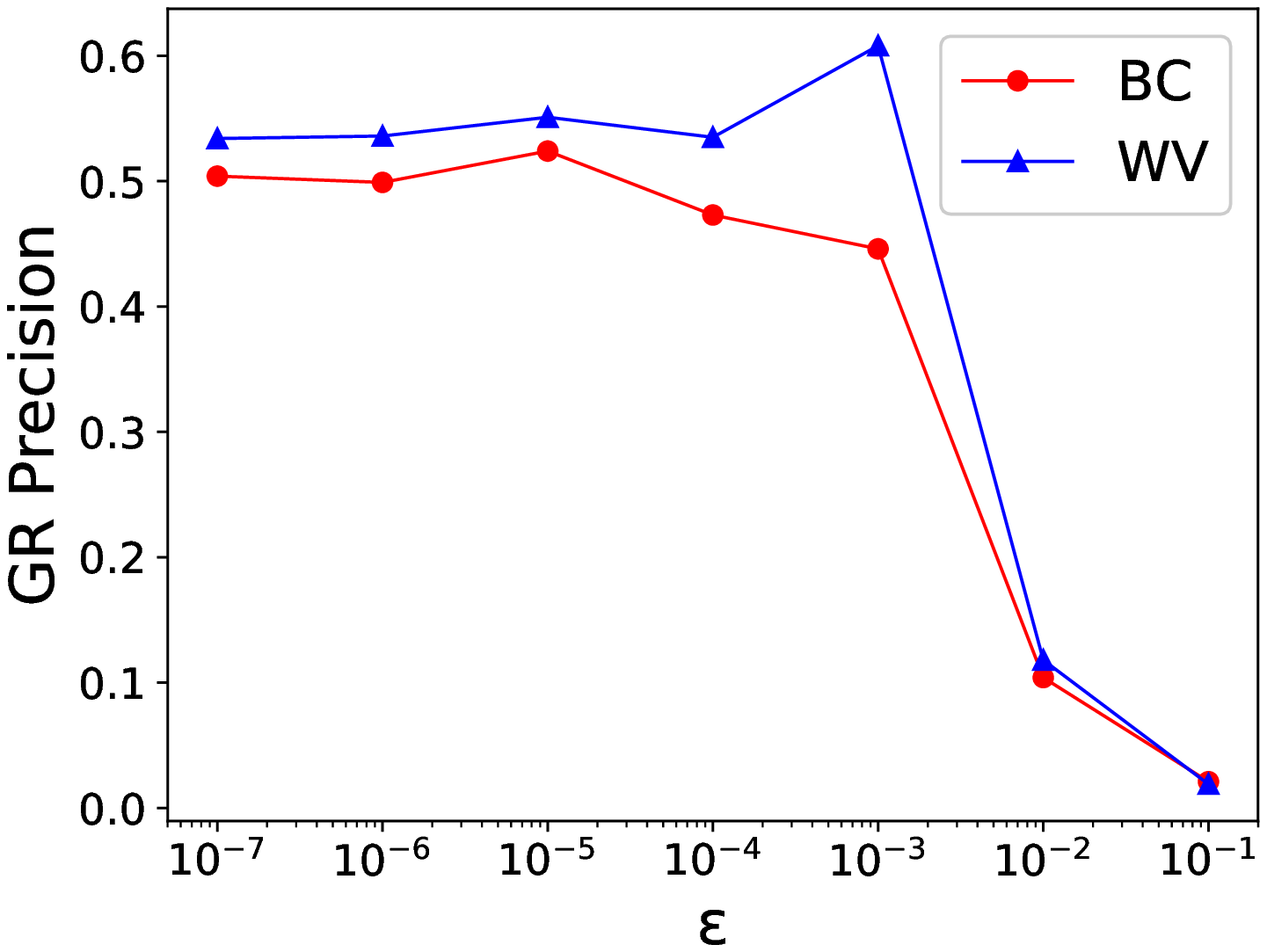}
  \vspace{-4mm}
 \caption{Running time and graph reconstruction precisions with varying $\e$.} \label{fig:parameter_eps}
\vspace{-3mm}
\end{small}
\end{figure}

\vspace{-2mm}
\section{Conclusion}
In this paper, we propose transpose proximity, a unified approach that
allows graph embeddings to preserve both in- and out-degree
distributions on directed graphs and to avoid the conflicting
optimization goals on undirected
graphs. Based on the concept of transpose proximity, we present \strap,
a factorization method that achieves both
scalability and non-linearity on large graphs. The theoretical analysis shows that
the running time of our algorithm is linear to the number of edges in
the graph. The experimental results show by using transpose proximity,
\strap outperforms competitors in both transductive and inductive
tasks, while achieving satisfying scalability. % On the other hand, due to its non-linear nature
% \strap achieves comparable or better performance  as
% the best random walk methods do in induction tasks such as link
% prediction and node classification.
As future work, an interesting
open problem is to study how to combine  transpose proximity with the skip-gram
model for better parallelism and predictive strength.

\vspace{-2mm}
\section{ACKNOWLEDGEMENTS}
This research was supported in part by National Natural Science Foundation of China
(No. 61832017 and No. 61732014) and by the Fundamental Research Funds for the Central Universities and the Research Funds of Renmin University of China under Grant 18XNLG21.

%%% Local Variables:
%%% mode: latex
%%% TeX-master: "paper"
%%% End:

%\input{conclusions.tex}

\vspace{-2mm}
\begin{small}
\bibliographystyle{plain}
\bibliography{paper}
\end{small}

%\appendix
%\input{appendix.tex}

\end{document}